\documentclass{article}
\PassOptionsToPackage{numbers,sort&compress}{natbib}

\usepackage[arxiv]{optional}

\opt{conf}{\usepackage{aistats2019}}
\opt{arxiv}{\usepackage[preprint]{arxiv_nips_2018}}
\usepackage[utf8]{inputenc} 
\usepackage[T1]{fontenc}    
\usepackage{hyperref}       
\usepackage{url}            
\usepackage{booktabs}       
\usepackage{amsfonts, amsmath}       
\usepackage{nicefrac}       
\usepackage{microtype}      
\usepackage{siunitx}
\usepackage{color}
\usepackage{algorithm}
\usepackage{hyperref}
\usepackage[noend]{algpseudocode}
\usepackage{float}
\usepackage{longtable}
\usepackage{xspace}
\makeatletter
\def\BState{\State\hskip-\ALG@thistlm}
\makeatother

\DeclareMathOperator\conv{ConvexHull} 
\DeclareMathOperator\dis{d} 

\newcommand{\R}{\mathbb{R}}
\newcommand{\defined}{:=}

\DeclareMathOperator*{\argmin}{argmin}
\newcommand{\E}{\mathbb{E}}
\newcommand{\Oh}{O}
\newcommand{\bigO}[1]{O \left ( #1 \right)}
\newcommand{\pr}[1]{\mathbb{P}\left( #1 \right)}
\newcommand{\Jup}{{J_{+}}}
\newcommand{\Juppow}[1]{{J_{+}^{#1}}}
\newcommand{\Jdown}{{J}}

\newcommand{\cdelta}{{c^*_\delta}}

\newcommand{\Zup}{{Z_{+}}}
\newcommand{\zup}{{z_{+}}}
\newcommand{\Zdown}{{Z}}

\newcommand{\err}{\epsilon}

\newcommand{\inprod}[2]{\left \langle #1 , #2 \right \rangle}
\def\Nystrom{Nystr\"om\xspace}
\def\JL{Johnson--Lindenstrauss\xspace}
\usepackage [autostyle, english = american]{csquotes}
\MakeOuterQuote{"}

\usepackage{amsmath,amsfonts,amssymb,amsthm, array, epsfig, epstopdf, url}
\usepackage{mathtools}
\usepackage{color}

\usepackage[capitalize, sort]{cleveref}
\usepackage{thmtools}
\theoremstyle{plain}
\newtheorem{nthm}{Theorem}[section]
\newtheorem{nprop}[nthm]{Proposition}
\newtheorem{nlem}[nthm]{Lemma}
\newtheorem{ncor}[nthm]{Corollary}

\theoremstyle{definition}
\newtheorem{ndefn}[nthm]{Definition}

\newtheorem{nassum}[nthm]{Assumption}

\theoremstyle{remark}

\newtheorem*{rmk}{Remark}

\crefname{nlem}{Lemma}{Lemmas}
\crefname{nprop}{Proposition}{Propositions}
\crefname{ncor}{Corollary}{Corollaries}
\crefname{nthm}{Theorem}{Theorems}
\crefname{nexa}{Example}{Examples}
\crefname{ndefn}{Definition}{Definitions}
\crefname{nassum}{Assumption}{Assumptions}
\crefformat{footnote}{#1\footnotemark[#2]#3}

\opt{arxiv}{
\title{Data-dependent compression of random features \\ for large-scale kernel approximation}
\author{
Raj Agrawal \\
CSAIL \\
Massachusetts Institute of Technology \\
\texttt{r.agrawal@csail.mit.edu}
\And
Trevor Campbell \\
Department of Statistics \\
University of British Columbia  \\
\texttt{trevor@stat.ubc.ca}
\AND
Jonathan H.~Huggins \\
Department of Biostatistics \\
Harvard University \\
\texttt{jhuggins@mit.edu}
\And
Tamara Broderick \\
CSAIL \\
Massachusetts Institute of Technology \\
\texttt{tbroderick@csail.mit.edu}
}}

\begin{document}

\opt{arxiv}{\maketitle}
\opt{conf}{
\runningtitle{Data-dependent compression of random features \\ for large-scale kernel approximation}%

\twocolumn[

\aistatstitle{Scalable Gaussian Process Inference with \\ Finite-data Mean and Variance Guarantees}

\aistatsauthor{ Author 1 \And Author 2 \And  Author 3 }
\aistatsaddress{ Institution 1 \And  Institution 2 \And Institution 3 }

]
}

%

%






\begin{abstract}
  Kernel methods offer the flexibility to learn complex relationships in modern, large data sets while enjoying strong theoretical guarantees on quality. Unfortunately, these methods typically require cubic running time in the data set size, a prohibitive cost in the large-data setting. Random feature maps (RFMs) and the \Nystrom method both consider low-rank approximations to the kernel matrix as a potential solution. But, in order to achieve desirable theoretical guarantees, the former may require a prohibitively large number of features $\Jup$, and the latter may be prohibitively expensive for high-dimensional problems. We propose to combine the simplicity and generality of RFMs with a data-dependent feature selection scheme to achieve desirable theoretical approximation properties of \Nystrom with just $O(\log \Jup)$ features. Our key insight is to begin with a large set of random features, then reduce them to a small number of weighted features in a data-dependent, computationally efficient way, while preserving the statistical guarantees of using the original large set of features. We demonstrate the efficacy of our method with theory and experiments---including on a data set with over 50 million observations. In particular, we show that our method achieves small kernel matrix approximation error and better test set accuracy with provably fewer random features than state-of-the-art methods.  
\end{abstract}


\section{Introduction}

Kernel methods are essential to the machine learning and statistics toolkit because of their modeling flexibility, ease-of-use,
and widespread applicability to problems including regression, classification, clustering, dimensionality reduction, 
and one and two-sample testing~\citep{Hofmann:2008,kern_book,Chwialkowski:2016,Gretton:2012}.
In addition to good empirical performance, kernel-based methods come equipped with strong statistical and learning-theoretic 
guarantees~\citep{Vapnik:1998,Mendelson:2003,Balcan:2008,large_margin, svm,Sriperumbudur:2010}.
Because kernel methods are nonparametric, they are particularly attractive for large-scale problems, where they make it possible 
to learn complex, highly non-linear structure from data. 
Unfortunately, their time and memory costs scale poorly with data size.
Given $N$ observations, storing the kernel matrix $K$ requires $O(N^{2})$ space.
Using $K$ for learning typically requires $O(N^{3})$ time,
as this often entails inverting $K$ or computing its singular value decomposition.

To overcome poor scaling in $N$, researchers have devised various approximations to exact kernel methods.
A widely-applicable and commonly used tactic is to replace $K$ with a rank-$J$ approximation,
which reduces storage requirements to $O(NJ)$ and computational complexity of inversion or singular value decomposition to $O(NJ^{2})$~\citep{Halko:2011}. 
Thus, if $J$ can be chosen to be constant or slowly increasing in $N$, only (near-)linear time and space 
is required in the dataset size. 
Two popular approaches to constructing low-rank approximations are random feature maps (RFMs)~\citep{dot_prod_kernel,spherical_rf,compo_kernels,Samo:2015}---particularly random 
Fourier features (RFFs)~\citep{rahimi_rf}---and \Nystrom-type approximations~\citep{nystroem}. 
The \Nystrom method is based on using $J$ randomly sampled columns from $K$, and thus is data-dependent.
The data-dependent nature of \Nystrom methods can provide statistical guarantees even when $J \ll N$, 
but these results either apply only to kernel ridge regression~\citep{Alaoui:2015,Yang:2017,Rudi:2015}
or require burdensome recursive sampling schemes~\citep{Musco:2017,Lim:2018}.
Random features, on the other hand, are simple to implement and use $J$ random features that are data-independent. For problems with both large $N$ and number of covariates $p$, an extension of random features called \emph{Fast Food RFM} has been successfully applied at a fraction of the computational time required by \Nystrom-type approximations, which are \emph{exponentially} more costly in terms of $p$ \citep{fast_food}. 
The price for this simplicity and data-independence is that a large number of random features is often needed to approximate the kernel matrix well \citep{Honorio:2017,dot_prod_kernel,rahimi_rf,Yang:2012, dnn_vs_rf}.   

The question naturally arises, then, as to whether we can combine the simplicity of random features and the ability to scale to large-$p$ problems 
with the appealing approximation and statistical properties of \Nystrom-type approaches. 
We provide one possible solution by making random features data-dependent, and we show promising theoretical and empirical results.  
Our key insight is to begin with a large set of random features, then reduce them to a small set of 
weighted features in a data-dependent, computationally efficient way, while preserving 
the statistical guarantees of using the original large set.
We frame the task of finding this small set of features as an optimization problem,
which we solve using ideas from the coreset literature~\citep{hilb_coresets,giga}.
Using greedy optimization schemes such as the Frank--Wolfe algorithm,
we show that a large set of $\Jup$ random features can be compressed to an
\emph{exponentially smaller} set of just $O(\log \Jup)$ features 
while still achieving the same statistical guarantees as using all $\Jup$ features. 
We demonstrate that our method achieves superior performance to existing approaches on a range of real datasets---including 
one with over 50 million observations---in terms of kernel matrix approximation and classification accuracy.

\section{Preliminaries and related work} \label{sec:prelims}

Suppose we observe data $\{(x_n, y_n)\}_{n=1}^N$ with predictors $x_n\in \R^p$ and
responses $y_n\in\R$. In a supervised learning task, we aim to find a model $f:\R^p \to \R$ among a set of 
candidates $\mathcal{F}$ that predicts the response well for new predictors. 
Modern data sets of interest often reach $N$ in the tens of millions or higher, allowing 
analysts to learn particularly complex relationships in data. 
Nonparametric kernel methods \citep{kern_book} offer 
a flexible option in this setting; by taking $\mathcal{F}$ to be 
a reproducing kernel Hilbert space with positive-definite kernel $k:\R^p\times \R^p \to \R$, they enable
learning more nuanced details of the model $f$ as more data are obtained. 
As a result, kernel methods are widespread not just in regression and classification
but also in dimensionality reduction, conditional independence testing, one and two-sample testing,
and more~\citep{kern_pca,kern_ci,kern_test,Gretton:2012,Chwialkowski:2016}.

The problem, however, is that kernel methods become computationally intractable for large $N$. 
We consider kernel ridge regression as a prototypical example~\citep{ridge_regres}.
Let $K \in \R^{N\times N}$ be the kernel matrix 
consisting of entries $K_{nm} \defined k(x_n, x_m)$. Collect the responses into the vector $y\in\R^N$. 
Then kernel ridge regression requires solving
\[
\min_{\alpha\in\R^N} -\frac{1}{2}\alpha^T(K+\lambda I)\alpha + \alpha^Ty,
\]
where $\lambda > 0$ is a regularization parameter. 
Computing and storing $K$ alone has $O(N^2)$ complexity, while
computing the solution
$\alpha^\star = (K+\lambda I)^{-1}y$ further requires solving a linear system, with cost $O(N^3)$. 
Many other kernel methods have $O(N^3)$ dependence; see \cref{tab:impact_frob_norm}.   

To make kernel methods tractable on large datasets, a common practice is to 
replace the kernel matrix $K$ with an approximate low-rank factorization $\hat{K}:= ZZ^T \approx K$, where
$Z\in\R^{N\times J}$ and $J\ll N$. 
This factorization can be viewed as replacing the kernel function
$k$ with a finite-dimensional inner product $k(x_n, x_m) \approx z(x_n)^Tz(x_m)$ between features
generated by a \emph{feature map} $z : \R^p \to \R^J$. 
Using this type of approximation significantly reduces downstream training time, as shown in 
the second column of \cref{tab:impact_frob_norm}. 
Previous results show that as long as $ZZ^T$ is close to $K$ in the Frobenius norm, 
the optimal model $f$ using $\hat{K}$ is uniformly close to the one using $K$~\citep{cortes_kern_approx};
see the rightmost column of \cref{tab:impact_frob_norm}.

\begin{table*} 
\caption{A comparison of training time for PCA, SVM, and ridge regression using the exact kernel matrix $K$ versus a low-rank 
approximation $\hat{K} = ZZ^T$, where $Z$ has $J$ columns. 
Exact training requires either inverting or computing the SVD of the true kernel matrix $K$ at a cost of $O({N^3})$ time, as shown in the first column. The second column refers to training the methods using a low-rank factorization $Z$. For ridge regression and PCA, the low-rank training cost reflects the time to compute and invert the feature covariance matrix $Z^TZ$. For SVM, the time refers to fitting a linear SVM on $Z$ using \emph{dual-coordinate descent} with optimization tolerance $\rho$ \citep{lin_svm_train}. The third column quantifies the uniform error between the function fit using $K$ and the function fit using $Z$. For specific details of how the bounds were derived, see \cref{A:downstream}.}
  \label{tab:impact_frob_norm}
  \centering
  \begin{tabular}{llll}
    \toprule
    \cmidrule(r){1-2}
    Method  & Exact Training Cost   & Low-Rank Training Cost   & Approximation Error \\
    \midrule
    PCA& $O(N^3)$ & $\Theta(NJ^2)$  & $O\left((1 - \frac{\ell}{N} ) \|\hat{K} - K \|_F\right)$     \\
    SVM& $O(N^3)$ & $\Theta(NJ\log \frac{1}{\rho})$ & $O \left(\|\hat{K} - K \|_F^{\frac{1}{2}}\right)$    \\
    Ridge Regression& $O(N^3)$ & $\Theta(NJ^2)$ & $O\left(\frac{1}{N}\|\hat{K} - K \|_F\right)$      \\
    \bottomrule
  \end{tabular}
\end{table*}

However, finding a good feature map is a nontrivial task. One popular method, known as
\emph{random Fourier features} (RFF)~\citep{rahimi_rf}, is based on Bochner's Theorem:
\begin{nthm}[{\citep[p.\ 19]{bochner}}] \label{bochner} 
A continuous, stationary kernel 
$k(x, y) = \phi(x - y)$ for $x, y \in \R^p$ is positive definite with $\phi(0) = 1$ if and 
only if there exists a probability measure $Q$ such that
\begin{equation}
\begin{split}
\phi(x - y) &= \int_{\R^p} e^{i\omega^T(x-y)} \mathrm{d} Q(\omega) \\
&= \E_{Q} [\psi_{\omega}(x) \psi_{\omega}(y)^{*}], \quad \psi_{\omega}(x) := e^{i\omega^Tx}.
\end{split}
\end{equation} 
\end{nthm}
\Cref{bochner} implies that 
$z_{\text{complex}}(x) \defined (\nicefrac{1}{\sqrt{J}})[\psi_{\omega_1}(x), \cdots, \psi_{\omega_J}(x)]^T$, 
where $\omega_i \overset{\text{i.i.d.}}{\sim} Q$, 
provides a Monte-Carlo approximation of the true kernel function. As noted by \citet{recht_rf}, the real-valued feature map $z(x) \defined (\nicefrac{1}{\sqrt{J}})[\cos(\omega_1^T x + b_1) , \cdots, \cos(\omega_J^T x + b_J)]^T$, $b_j \overset{\text{unif.}}{\sim} [0, 2\pi]$ also yields an unbiased estimator of the kernel function; we use this feature map in what follows unless otherwise stated. 
The resulting $N\times J$ feature matrix $Z$ yields estimates of the true kernel function with standard Monte-Carlo error rates of $\bigO{\nicefrac{1}{\sqrt{J}}}$
uniformly on compact sets~\citep{rahimi_rf,error_rand_feats}. 
The RFF methodology also applies quite broadly. There are
well-known techniques for obtaining samples from $Q$ for a variety of popular kernels such as 
the squared exponential, Laplace, and Cauchy~\citep{rahimi_rf}, as well as extensions 
to more general \emph{random feature maps}~(RFMs), which apply to many types of non-stationary 
kernels~\citep{dot_prod_kernel, spherical_rf, compo_kernels}. 

The major drawback of RFMs is the $\Oh(NJp)$ time and $\Oh(NJ)$ memory costs associated
with generating the feature matrix $Z$.\footnote{\emph{Fast Food RFM} can reduce the computational cost of generating the feature matrix to $\Oh(NJ \log p)$ by exploiting techniques from sparse linear algebra. For simplicity, we focus on RFM here, but we note that our method can also be used on top of Fast Food RFM in cases when $p$ is large.} Although these are linear in $N$ as desired,
recent empirical evidence \citep{dnn_vs_rf} suggests that $J$ needs to be quite large 
to provide competitive performance with other data analysis techniques.
Recent work addressing this drawback has broadly involved two approaches: 
\emph{variance reduction} and \emph{feature compression}. 
Variance reduction techniques involve modifying the standard 
Monte-Carlo estimate of $k$, e.g.\ with control variates, quasi-Monte-Carlo techniques, or importance sampling 
\citep{quasi_rf, stein_rf, moment_rf, orthog_rf,Musco}. These approaches either depend poorly on the 
data dimension $p$ (in terms of statistical generalization error), or, for a fixed approximation error, reduce the number of features $J$ 
compared to RFM only by a constant. Feature compression techniques, on the other hand,
involve two steps: (1) ``up-projection,'' in which the basic RFM methodology generates a 
large number $\Jup$ of features---followed by (2) ``compression,'' in which those features are used to find
a smaller number $\Jdown$ of features while ideally retaining the kernel approximation error
of the original $\Jup$ features. Compact random feature maps \citep{compact_rf}
represent an instance of this technique in which compression is achieved using the 
\JL (JL) algorithm~\citep{JL_algo}. 
However, not only is the generation and storage of $\Jup$ features prohibitively expensive
for large datasets, JL compression is \emph{data-independent} and leads to only a constant reduction 
in $\Jup$ as we show in \cref{A:runtime_analysis}~(see summary in \cref{tab:theo_comp}).  


\section{Random feature compression via coresets} \label{sec:construct_rand_feats}

In this section, we present an algorithm for approximating a kernel matrix $K\in\R^{N\times N}$
with a low-rank approximation $K\approx \hat{K} = \Zdown\Zdown^T$ 
obtained using a novel feature compression technique. In the up-projection step
we generate $\Jup$ random features, but only compute their values for a small, randomly-selected 
subset of $S \ll N^2$ datapoint pairs.
In the compression step, we select a sparse, weighted subset of $\Jdown$ of the original $\Jup$ features in a sequential
greedy fashion. We use the feature values on the size-$S$ subset of all possible data pairs to decide, at each step, which feature to 
include and its weight.
Once this process is complete, we compute the resulting weighted subset of $\Jdown$ features on the whole dataset. We use this low-rank approximation of the kernel in our original learning problem. Since we use a sparse weighted
feature subset for compression---as opposed to a general linear combination as in previous work---we do not need to compute
all $\Jup$ features for the whole dataset. This circumvents the expensive $O(N\Jup p)$ up-projection computation typical of 
past feature compression methods. In addition, 
we show that our greedy compression algorithm needs to output only $\Jdown=O(\log \Jup)$
features---as opposed to past work, where $\Jdown=O(\Jup)$ was required---while maintaining the same kernel approximation error 
provided by RFM with $\Jup$ features. These results are summarized in \cref{tab:theo_comp}
and discussed in detail in \cref{sec:theo_analysis}.

\subsection{Algorithm derivation}\label{sec:alg_derivation}
Let $\Zup\in\R^{N\times \Jup}$, $\Jup > \Jdown$, 
be a fixed up-projection feature matrix generated by RFM. Our goal is to use $\Zup$ 
to find a compressed low-rank approximation $\hat{K} = \Zdown\Zdown^T \approx K$, $\Zdown\in\R^{N\times \Jdown}$.
Our approach is motivated by the fact that spectral 2-norm bounds on $K-\hat{K}$ provide
uniform bounds on the difference between learned models using $K$ and 
$\hat{K}$ \citep{cortes_kern_approx}, 
as well as the fact that the Frobenius norm
bounds the 2-norm. So we aim to find a $\Zdown$
that minimizes the Frobenius norm error $\|K-\Zdown\Zdown^T\|_F$. 
By the triangle inequality, 
\begin{align}
\lefteqn{\|K - \Zdown\Zdown^T\|_F} \nonumber \\
&\leq \|K - \Zup\Zup^T\|_F + \|\Zup\Zup^T - \Zdown\Zdown^T\|_F, \label{eq:tri_ineq}
\end{align}
so constructing a good feature compression down to $J$ features amounts 
to picking $\Zdown$ such that $\Zup\Zup^T\approx \Zdown\Zdown^T$ in Frobenius norm.
Let $\Zup_{j} \in \R^{N}$ denote the $j$th column of $\Zup$. 
Then we would ideally like to solve the optimization problem
\begin{equation} \label{combo_opt}
\begin{split}
\argmin_{w \in \R_+^{\Jup}} \quad & \frac{1}{N^2} \|\Zup\Zup^T - Z(w)Z(w)^T\|^2_F\\
\quad \text{s.t.}  \quad Z(w) &:= \left[\begin{array}{ccc}
                                      \sqrt{w_1} \Zup_1 & \cdots & \sqrt{w_{\Jup}}\Zup_{\Jup}
                                 \end{array}\right] \\ &  \|w\|_0 \leq J.
\end{split}
\end{equation}
This problem is intractable to solve exactly for two main reasons. First, computing the objective function 
requires computing $\Zup$, which itself takes $\Omega(N\Jup p)$ time. 
But it is not uncommon for all three of $N$, $\Jup$, and $p$ to be large, making this computation expensive.
Second, the cardinality, or ``0-norm,'' constraint on $w$ yields a difficult combinatorial optimization.
In order to address these issues, first note that 
\begin{equation*}
    \begin{split}
       & \frac{1}{N^2}\|\Zup\Zup^T - Z(w)Z(w)^T\|^2_F = \\
        & \E_{i,j \overset{\text{i.i.d.}}{\sim} \pi}\left[(\zup_i^T\zup_j - z_i(w)^Tz_j(w))^2\right], 
    \end{split}
\end{equation*}
where $\pi$ is the uniform distribution on the integers $\{1, \dots, N\}$, and $\zup_i, z_i(w) \in \R^{\Jup}$ 
are the $i$th rows of $\Zup$, $Z(w)$, respectively. Therefore, we can generate a 
Monte-Carlo estimate of the optimization objective by sampling $S$ pairs $i_s, j_s \overset{\text{i.i.d.}}{\sim} \pi$:
\begin{equation} \label{eq:sub_sample}
    \begin{split}
        & \frac{S}{N^2}\|\Zup\Zup^T - Z(w)Z(w)^T\|^2_F \\
        & \approx \sum_{s=1}^{S}(\zup_{i_s}^T\zup_{j_s} - z_{i_s}(w)^Tz_{j_s}(w))^2 \\
        & = (1-w)^TRR^T(1-w) \ \text{s.t.}
    \end{split}
\end{equation}
\begin{equation*}
R := \left[\begin{array}{ccc}
            \zup_{i_1} \circ \zup_{j_1}, & \cdots, & \zup_{i_S} \circ \zup_{j_S}
            \end{array}\right] \in \R^{\Jup \times S},
\end{equation*}
where $\circ$ indicates a component-wise product. Denoting the $j$th row of $R$ by $R_j \in \R^{S}$
and the sum of the rows by $r = \sum_{j=1}^{\Jup} R_{j}$, 
we can rewrite the Monte Carlo approximation of the original optimization problem in \cref{combo_opt} as
\begin{equation} \label{combo_opt_projected}
    \begin{split}
        & \argmin_{w \in \R_+^{\Jup}} \quad \| r - r(w)\|_2^2 \\
        & \text{s.t.} \quad \|w\|_0 \leq \Jdown,
    \end{split}
\end{equation}
where
$r(w) := \sum_{j=1}^{\Jup}w_j R_j$.
Note that the $s^\text{th}$ component $r_s = \zup_{i_s}^T \zup_{j_s}$ of $r$ is the Monte-Carlo approximation of $k(x_{i_s}, x_{j_s})$ using all $\Jup$ features, while $r(w)_s = (\sqrt{w} \circ \zup_{i_s})^T (\sqrt{w} \circ \zup_{j_s})$ is the sparse Monte-Carlo approximation using weights $w\in\R_+^{\Jup}$. In other words, the difference between the full optimization in  \cref{combo_opt} and the reformulated optimization in \cref{combo_opt_projected} is that the former attempts to find a sparse, weighted set of features that approximates the full $\Jup$-dimensional feature inner products for all data pairs, while the latter attempts to do so only for the subset of pairs $i_s, j_s$, $s\in\{1, \dots, S\}$. Since a kernel matrix is symmetric and $k(x_n, x_n) = 1$ for any datapoint $x_n$, we only need to sample $(i, j) $ above the diagonal of the $N \times N$ matrix (see \cref{algfeatcompress}).

The reformulated optimization problem in \cref{combo_opt_projected}---i.e., approximating the sum $r$ of a 
collection $(R_j)_{j=1}^{\Jup}$ of vectors in $\R^S$ with a sparse weighted linear combination---is precisely the 
\emph{Hilbert coreset construction problem} studied in previous work \citep{hilb_coresets,giga}. 
There exist a number of efficient algorithms to solve this problem approximately; 
in particular, the Frank--Wolfe-based method of \citet{hilb_coresets} and ``greedy iterative geodesic ascent'' (GIGA) \citep{giga}
both provide an exponentially decreasing objective value as a function of the compressed number of features $\Jdown$.
Note that it is also possible to apply other more general-purpose methods for cardinality-constrained 
convex optimization \citep{atom_pursuit,dantzig,lasso}, but these techniques are 
often too computationally expensive in the large-dataset setting.
Our overall algorithm for feature compression is shown in \cref{algfeatcompress}.    

\begin{algorithm} 
\caption{Random Feature Maps Compression (RFM-FW / RFM-GIGA)}\label{algfeatcompress}
\hspace*{\algorithmicindent} \textbf{Input:} Data $(x_n)_{n=1}^N$ in $\R^p$, RFM distribution $Q$, number of starting random features $\Jup$, number of compressed features $\Jdown$, number of data pairs $S$ \\
\hspace*{\algorithmicindent} \textbf{Output:} Weights $w \in \R^{\Jup}$ with at most $\Jdown$ non-zero entries 
\begin{algorithmic}[1]
\State $(i_s, j_s)_{s=1}^S \overset{\text{i.i.d.}}{\sim} \mathrm{Unif}\left( \{(i, j): i < j, 2 \leq j \leq N \} \right)$.
\State Sample $(\omega_j)_{j=1}^{\Jup} \overset{\text{i.i.d.}}{\sim} Q$ 
\State Sample $b_j \overset{\text{unif.}}{\sim} [0, 2\pi ], 1 \leq j \leq \Jup$
\For{$s=1:S$}
\State Compute $\zup_{i_s} \gets (\nicefrac{1}{\sqrt{\Jup}})[\cos(\omega_1^T x_{i_s} + b_1) , \cdots, \cos(\omega_{\Jup}^T x_{i_s} + b_{\Jup})]^T$; same for $\zup_{j_s}$
\EndFor
\State Compute $R \gets \left[\zup_{i_1} \circ \zup_{j_1}, \,\, \cdots, \,\, \zup_{i_S} \circ \zup_{j_S} \right]$
\State $R_j\gets$ row $j$ of $R$; $r \gets \sum_{j=1}^{\Jup} R_j$
\State $w \gets$ solution to \cref{combo_opt_projected} with FW \citep{hilb_coresets} or GIGA \citep{giga}
\State $Z(w) = \left[\begin{array}{ccc}
                                      \sqrt{w_1} \Zup_1 & \cdots & \sqrt{w_{\Jup}}\Zup_{\Jup} \end{array}\right]$  
\State \Return $Z(w)$
\end{algorithmic}
\end{algorithm}


\subsection{Theoretical results}\label{sec:theo_analysis}
In order to employ \cref{algfeatcompress}, we must choose the number $S$ of data pairs, the up-projected feature
dimension $\Jup$, and compressed feature dimension $\Jdown$. 
Selecting these three quantities involves a tradeoff between
the computational cost of using \cref{algfeatcompress} and the resulting low-rank kernel approximation Frobenius error,
but it is not immediately clear how to perform that tradeoff. 
\cref{thm:approx_norm,cor:approx_norm} provide a remarkable resolution to this issue:
roughly, if we fix $\Jup$ such that the basic random features method
provides kernel approximation error $\epsilon > 0$ with high probability, then 
choosing $S = \Omega(\Juppow{2}(\log \Jup)^2)$ and $\Jdown = \Omega(\log\Jup)$ suffices to guarantee that
the compressed feature kernel approximation error is also $O(\epsilon)$ with high probability.
In contrast, previous feature compression methods required $\Jdown = \Omega(\Jup)$ to achieve the same result; see \cref{tab:theo_comp}.
Note that \cref{thm:approx_norm} assumes that the compression step in \cref{algfeatcompress} is completed
using the Frank--Wolfe-based method from \citet{hilb_coresets}. However, this choice was made solely to 
simplify the theory; as GIGA \citep{giga} provides stronger performance both theoretically and empirically,
we expect a stronger result than \cref{thm:approx_norm,cor:approx_norm} to hold when using GIGA.
The proof of \cref{thm:approx_norm} is given in \cref{A:proof_approx_inner} and depends on the following assumptions.
\begin{nassum}\label{assumptions} 
\begin{enumerate} 
\item[(a)] The cardinality of the set of vectors $\{x_i - x_j, x_i + x_j \}_{1 \leq i < j \leq N}$ is $\frac{N(N-1)}{2}$, i.e., all  vectors $x_i - x_j, x_i + x_j, 1 \leq i < j \leq N$  are distinct. 
\item[(b)] $Q(\omega)$ for $\omega \in \R^p$ has strictly positive density on all of $\R^p$,  where $Q$ is the measure induced by the kernel $k$; see \cref{bochner}. 
\end{enumerate}
\end{nassum}
\Cref{assumptions}(a-b) are sufficient to guarantee that the compression coefficient $\nu_{\Jup}$ provided in \cref{thm:approx_norm} does not go to 1.
If $\nu_{\Jup} \to 1$ as $\Jup \to \infty$, the amount of compression could go to zero asymptotically.  When the $x_j$'s contain continuous (noisy) measurements, \Cref{assumptions}(a) is very mild since the difference or sum between two datapoints is unlikely to equal the difference or sum between two other datapoints. \Cref{assumptions}(b) is satisfied by most kernels used in practice (e.g.\ radial basis function, Laplace kernel, etc.). 

We obtain the exponential compression in \cref{thm:approx_norm} for the following reason: Frank-Wolfe and GIGA converge linearly when the minimizer of \cref{combo_opt_projected} belongs to the relative interior of the feasible set of solutions \citep{franke_wolfe}, which turns out to occur in our case. With linear convergence, we need to run only a \emph{logarithmic} number of iterations (which upper bounds the sparsity of $w$) to approximate $r$ by $r(w)$ for a given level of approximation error. For fixed $\Jup$, Lemma A.5 from \citet{hilb_coresets} immediately implies that the minimizer belongs to the relative interior. As $\Jup \rightarrow \infty$ (that is, as we represent the kernel function exactly), we show that the minimizer asymptotically belongs to the relative interior, and we provide a lower bound on its distance to the boundary of the feasible set. This distance lower bound is key to the asymptotic  worst-case bound on the compression coefficient given in \cref{thm:approx_norm} and \cref{thm:asym_comp_coef}.   

\begin{table*} 
  \caption{A comparison of the computational cost of basic random feature maps (RFM), 
  RFM with JL compression (RFM-JL), and RFM with our proposed compression using FW (RFM-FW) 
  for $N$ datapoints and $\Jup = \nicefrac{1}{\epsilon}\log \nicefrac{1}{\epsilon}$ up-projection features. 
  The first column specifies the number of compressed features $\Jdown$ needed to retain the $O(\err)$ high probability kernel approximation error guarantee of RFM. 
  The second and third columns list the complexity for computing the compressed features
  and using them for PCA or ridge regression, respectively. Theoretically, the number of datapoint pairs $S$ should be set to $\Omega(\Jup^2 (\log \Jup)^2)$ in \cref{algfeatcompress} (see  \cref{thm:approx_norm}) but empirically we find in \cref{sec:experiments} that $S$ can be set much smaller. See \cref{A:runtime_analysis} for derivations.
  }
  \label{tab:theo_comp}
  \centering
  \begin{tabular}{llll}
    \toprule
    \cmidrule(r){1-2}
    Method     & \# Compressed Features $\Jdown$ &   Cost of Computing $\Zdown$   &  PCA/Ridge Reg.\ Cost    \\
    \midrule
    RFM & $O \left( \Jup \right)$ & $O \left( N\Jup \right)$  &  $O \left( N\Juppow{2} \right)$     \\
    
    RFM-JL     & $O \left( \Jup \right) $ & $O \left( N\Jup\log \Jup \right)$ &  $O \left( N\Juppow{2} \right)$     \\
    
    RFM-FW     & $O \left( \log \Jup \right)$ & $O\left( S\Jup \log \Jup + N \log \Jup  \right)$ &  $O \left( N(\log \Jup)^2 \right)$      \\
    \bottomrule
  \end{tabular}
\end{table*}

\begin{nthm} \label{thm:approx_norm}
Fix $\epsilon > 0$, $\delta \in (0, 1)$, and $\Jup\in\mathbb{N}$. Then there are constants $\nu_{\Jup} \in (0, 1)$, which depends only on $\Jup$, and $0 \leq \cdelta < \infty$, which depends only on $\delta$, such that if 
\begin{align*}
J = \Omega\left(-\frac{\log\Jup}{\log\nu_{\Jup}}\right) 
\ and \ 
S =  \Omega \left(\frac{\cdelta}{\err^2} \left[\frac{\log\frac{1}{\err}}{\log\nu_{\Jup}}\right]^4 \log\Jup \right),
\end{align*}
then with probability at least $1 - \delta$, the output $Z$ of \cref{algfeatcompress} satisfies
\begin{equation*}
\frac{1}{N^2}\|\Zup\Zup^T - \Zdown\Zdown^T \|_F^2 \leq \err.
\end{equation*}
Furthermore, the compression coefficient is asymptotically bounded away from 1. That is,
\begin{equation}
   0 < \limsup_{\Jup \rightarrow \infty} \nu_{\Jup} < 1.
\end{equation}
\end{nthm}
\begin{ncor}\label{cor:approx_norm}
In the setting of \cref{thm:approx_norm}, if we let $\Jup = \Omega(\nicefrac{1}{\epsilon}\log \nicefrac{1}{\epsilon} )$, then 
\begin{equation*}
\frac{1}{N^2}\|K - \Zdown\Zdown^T \|_F^2 = O(\err).
\end{equation*}
\end{ncor}
\begin{proof}
Claim 1 of \citet{rahimi_rf} implies that $\frac{1}{N^2}\|K - \Zup\Zup^T\|_F^2 = \bigO{\err}$ if we set $\Jup = \Omega(\nicefrac{1}{\epsilon}\log \nicefrac{1}{\epsilon})$. The result follows by combining \cref{thm:approx_norm,eq:tri_ineq}.
\end{proof}

\cref{tab:theo_comp} builds on the results of \cref{thm:approx_norm,cor:approx_norm} 
to illustrate the benefit of our proposed feature compression technique
in the settings of kernel principal component analysis (PCA) and ridge regression. Since random features and random features with JL compression
both have $\Jdown = \Omega(\Jup)$, the 
$O(N\Juppow{2})$ cost of computing the feature covariance matrix $\Zdown^T\Zdown$ dominates when training
PCA or ridge regression. In contrast, the dominant cost of random features with our proposed algorithm is the compression step;
each iteration of Frank-Wolfe has cost $O(\Jup S)$, and we run it for $O(\log \Jup)$ iterations.

While \cref{cor:approx_norm} says how large $S$ must be for a given $\Jup$, it does not say how to pick $\Jup$, or equivalently how to choose the level of precision $\epsilon$. As one would expect, the amount of precision needed depends on the downstream application. For example, recent theoretical work suggests that both kernel PCA and kernel ridge regression require $\Jup$ to scale only sublinearly with the number of datapoints $N$ to achieve the same statistical guarantees as an exact kernel machine trained on all $N$ datapoints \citep{kern_pca_asym, Musco, rudi_kernel_ridge}. For kernel support vector machines (SVMs), on the other hand, \citet{error_rand_feats} suggest that $\Jup$ needs to be larger than $N$. Such a choice of $\Jup$ would make random features \emph{slower} than training an exact kernel SVM. However, since \citet{error_rand_feats} do not provide a lower bound, it is still an open theoretical question how $\Jup$ must scale with $N$ for kernel SVMs. 

For $\Jup$ even moderately large, setting $S = \Omega(\Juppow{2}(\log \Jup)^2))$ to satisfy \cref{thm:approx_norm} will be prohibitively expensive. Fortunately, in practice,   
we find $S \ll \Juppow{2}$ suffices to provide
significant practical computational gains without adversely affecting approximation error;
see the results in \cref{sec:experiments}.
We conjecture that we see this behavior since we expect even a small number of data pairs $S$
to be enough to guide feature compression in a data-dependent manner. We empirically verify this intuition in \cref{fig:S_vs_error} of \cref{sec:experiments}.

Finally, we provide an asymptotic upper bound for the compression coefficient $\nu_{\Jup}$. We achieve greater compression 
when $\nu_{\Jup} \downarrow 0$. Hence, the upper bound below shows the asymptotic worst-case rate of compression.
\begin{nthm} \label{thm:asym_comp_coef}
Suppose all $\{(i, j): 1 \leq i < j \leq N\}$ are sampled in \cref{algfeatcompress}. Then, 
\begin{equation} \label{eq:asymp_comp_coef}
   0 < \limsup_{\Jup \rightarrow \infty} \nu_{\Jup} < 1 - \frac{\left(1 - \frac{\|K\|_F}{c_Q} \right)^2}{2} < 1, 
\end{equation}
where $K$ is the exact kernel matrix and 
\begin{equation} 
\begin{split}
     & c_Q \coloneqq \frac{1}{N} \E_{\omega \sim Q, b \sim \text{Unif}[0, 2\pi]} \| u(\omega, b) \|_2, \textrm{with} \\
     & u(\omega, b) \coloneqq (\cos(\omega^Tx_i + b) \cos(\omega^Tx_j + b))_{i,j \in [N]}.
\end{split}
\end{equation}
\end{nthm}
By \cref{bochner}, $\| K \|_F = \frac{1}{N} \| \E_{\omega, b} u(\omega, b) \|_2$, so $\| K \|_F \leq c_Q$ by Jensen's inequality. In \cref{app:proof_asym_comp_coef}, we show this inequality holds strictly.
Hence the term squared in \cref{eq:asymp_comp_coef} lies in $(0,1]$.  Recall $\|K\|_F^2 = \sum_{i=1}^N \lambda_i$, for $\lambda_i$ the eigenvalues of $K$. With these observations, \Cref{thm:asym_comp_coef} says that the asymptotic worst-case rate of compression improves if $K$'s eigenvalue sum is smaller. As rough intuition: If the sum is small, then $K$ may be nearly low-rank and thus easier to approximate via a low-rank approximation. Since we subsample only $S$ of all  pairs in \cref{thm:approx_norm}, the upper bound in \cref{thm:asym_comp_coef} does not necessarily apply. Nonetheless, for $S$ moderately large, this upper bound roughly characterizes the worst-case compression rate for \cref{algfeatcompress}.


\begin{table}
  \caption{All datasets are taken from \texttt{LIBSVM}.}\label{tab:datasets}
  \label{sample-table}
  \centering
  \begin{tabular}{llll}
    \toprule
    \cmidrule(r){1-2}
    Dataset     & $\#$ Samples     & Dimension     & $\#$ Classes  \\
    \midrule
    Adult & 48,842  & 123 & 2    \\
    Human 
    	& 10,299 & 561 & 6     \\
    MNIST     & 70,000 & 780 & 10     \\
    Sensorless     & 58,000 & 9 & 11     \\
        Criteo     & 51,882,752 & 1,000,000 & 2    \\
    \bottomrule
  \end{tabular}
\end{table}

\section{Experiments} \label{sec:experiments}

\begin{figure}
\begin{centering}
\includegraphics[width=.95\linewidth]{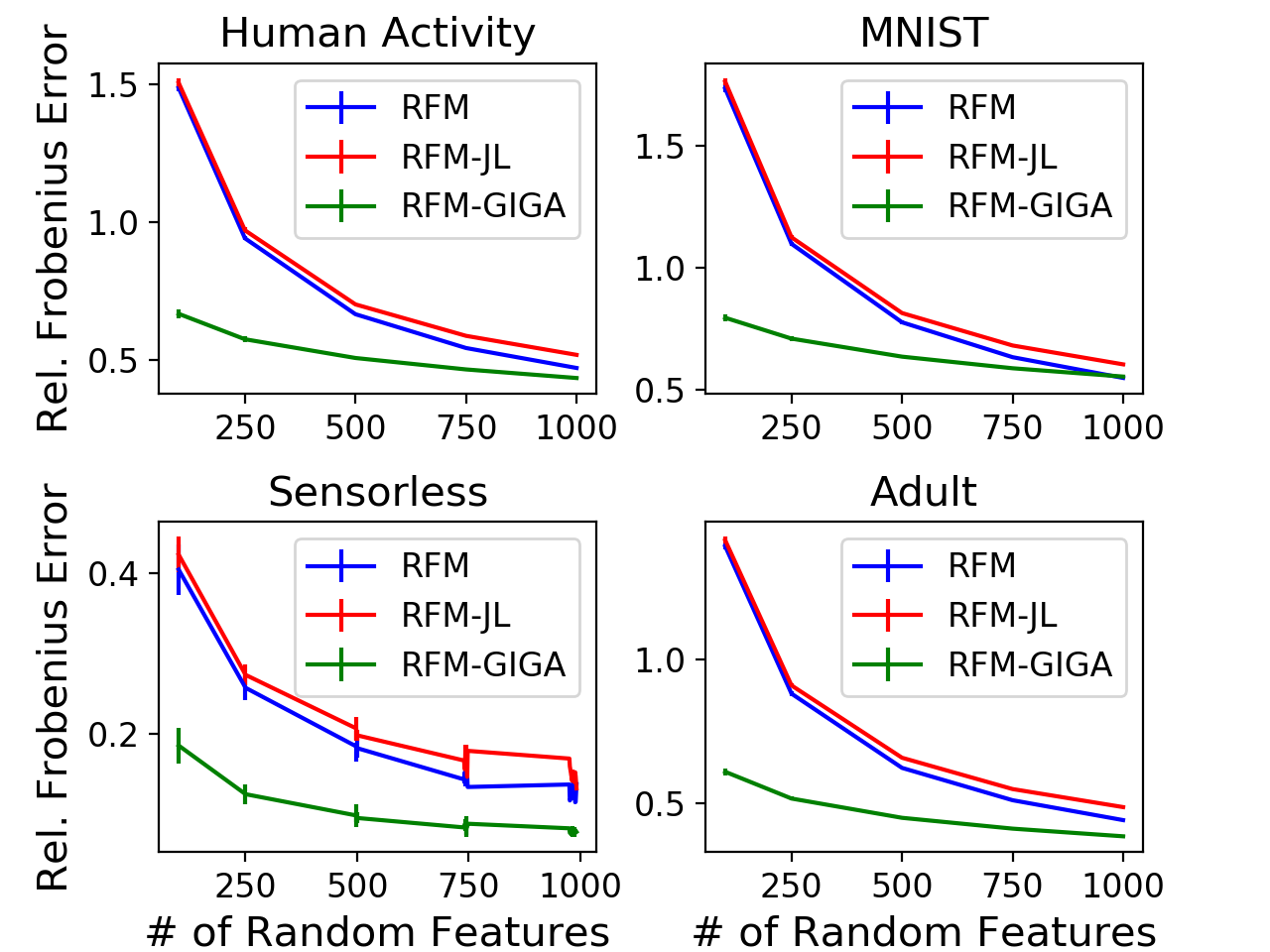}\par
\end{centering}
\caption{Kernel matrix approximation error.
Lower is better.
Points average 20 runs; error bar is one standard deviation.}
\label{fig:kern_errors}
\end{figure}

In this section we provide an empirical comparison 
of basic random feature maps (RFM) \citep{rahimi_rf},
RFM with Johnson-Lindenstrauss compression (RFM-JL) \citep{compact_rf},
and our proposed algorithm with compression via greedy iterative 
geodesic ascent \citep{giga} (RFM-GIGA). 
We note that there are many other random feature methods,
such as Quasi-Monte-Carlo random features \citep{quasi_rf}, that one might consider
besides RFM-JL. A strength of our method is that it can be used as an additional compression step
with these methods
and is thus complementary with them;
we discuss this idea and demonstrate the resulting improvements in \cref{A:add_experiments}.
In this section, we focus on Johnson-Lindenstrauss as the current state-of-the-art random features compression method.

\begin{figure}
\begin{centering}
\includegraphics[width=.95\linewidth]{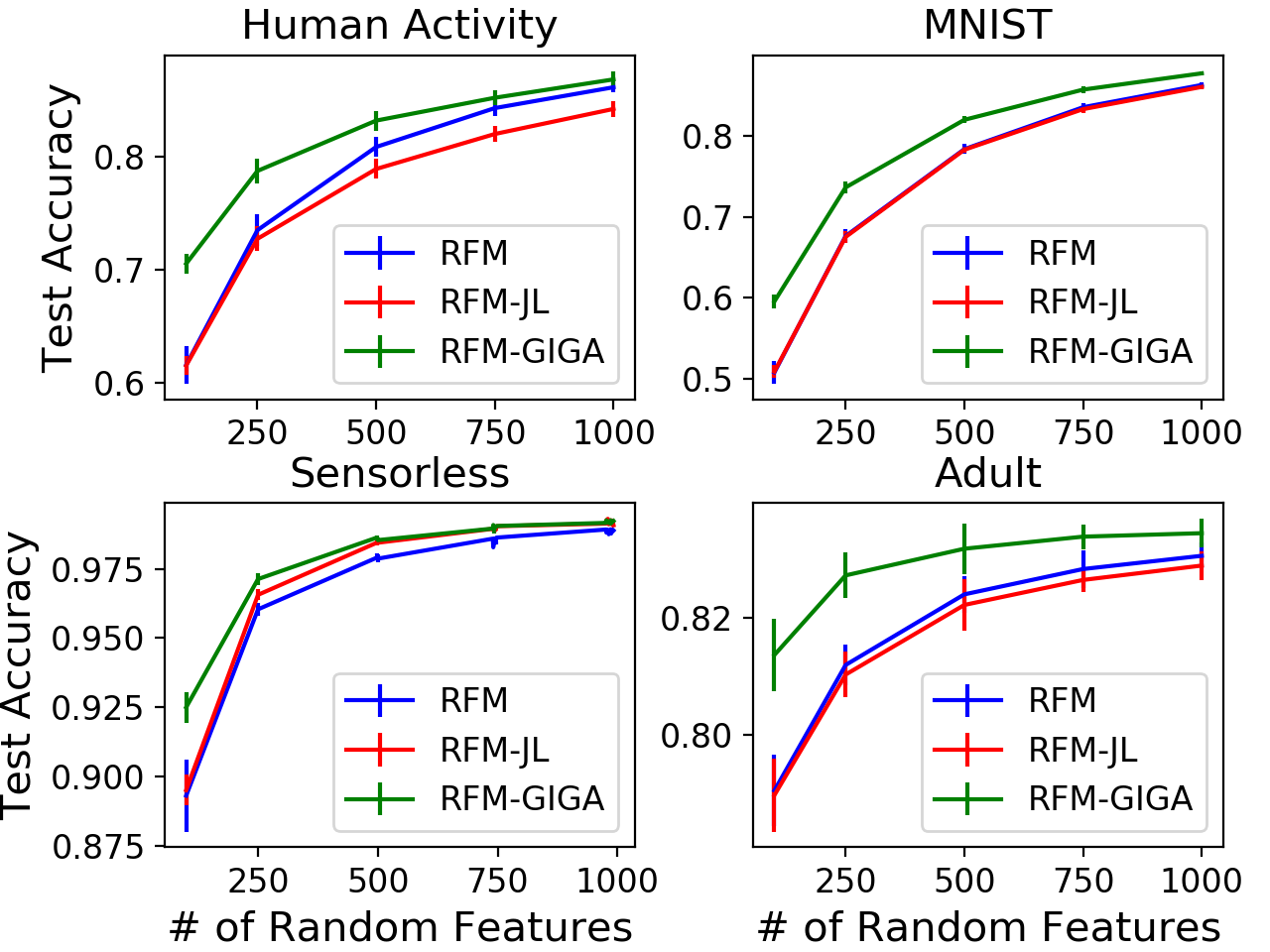}\par
\end{centering}
\caption{Classification accuracy.
Higher is better.
Points average 20 runs; error bar is one standard deviation.}
\label{fig:class_acc}
\end{figure}

We compare performance on the task of kernel SVM classification \citep{svm}. We consider five 
real, large-scale datasets, summarized in \cref{tab:datasets}.
We assess performance via two quality metrics---Frobenius error 
of the kernel approximation and test set classification error. We also measure
overall computation time---including both random feature projection and SVM training.
We use the radial basis kernel $k(x, y) = e^{-\gamma \|x - y\|^2}$; we 
pick both $\gamma$ and the SVM regularization strength for each dataset
by randomly sampling 10,000 datapoints, training an exact kernel
SVM on those datapoints, and using 5-fold cross-validation.
For both RFM-JL and RFM-GIGA we set $\Jup = 5{,}000$,
and for RFM-GIGA we set $S = 20{,}000$.

\cref{fig:kern_errors,fig:class_acc} show the relative kernel matrix approximation error $\|ZZ^T - K\|_F / \|K\|_F$ and test classification accuracy, respectively, as a function of the number of compressed features $\Jdown$. Note that, since we cannot actually compute $K$, we approximate the relative Frobenius norm error by randomly sampling $10^4$ datapoints. We ran each experiment 20 times; the results in \cref{fig:kern_errors,fig:class_acc} show the mean across these trials with one standard deviation denoted with error bars. RFM-GIGA outperforms RFM and RFM-JL across all the datasets, on both metrics, for the full range of number of compressed features that we tested. This empirical result corroborates the theoretical results presented earlier in \cref{sec:theo_analysis}; in practice, RFM-GIGA requires approximately an order of magnitude fewer features than either RFM or RFM-JL.

\begin{figure}
\begin{centering}
\includegraphics[width=.7\linewidth]{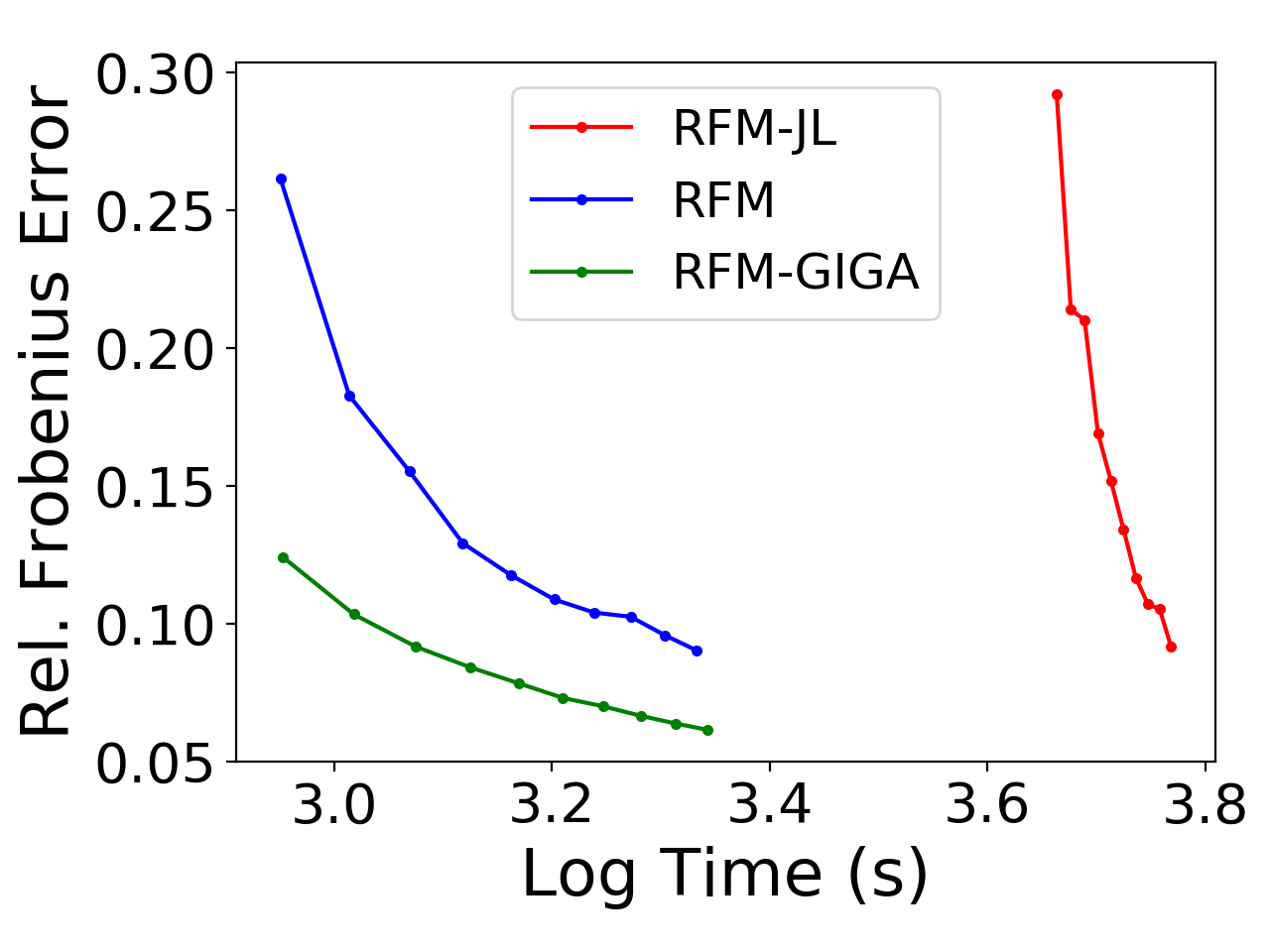}\par
\end{centering}
\caption{Log clock time vs.\ kernel matrix approximation quality on the Criteo data. Lower is better.}
\label{fig:criteo_kern}
\end{figure}

To demonstrate the computational scalability of RFM-GIGA, we also plot the relative kernel matrix approximation error versus computation time for the Criteo dataset, which consists of over 50 million data points. Before random feature projection and training, we used sparse random projections \citep{sparse_rand_projs} to reduce the input dimensionality to 250 dimensions (due to memory constraints). We set $\Jup = 5000$ and $S = 2 \times 10^4$ as before, and let $J$ vary between $10^2$ and $10^3$. The results of this experiment in \cref{fig:criteo_kern} suggest that RFM-GIGA provides a significant improvement in performance over both RFM and RFM-JL. Note that RFM-JL is very expensive in this setting---the up-projection step requires computing a $5 \times 10^8$ by $5\times 10^3$ feature matrix---explaining its large computation time relative to RFM and RFM-GIGA. For test-set classification, all the methods performed the same for all choices of $J$ (accuracy of $0.74 \pm 0.001$), so we do not provide the runtime vs.\ classification accuracy plot. This result is likely due to our compressing the $10^6$-dimensional feature space to 250 dimensions, making it hard for the SVM classifier to properly learn.


Given the empirical advantage of our proposed method, we next focus on understanding (1) if $S$ can be set much smaller than $\Omega(\Juppow{2}(\log \Jup)^2))$ in practice and (2) if we can get an exponential compression of $\Jup$ in practice as \cref{thm:approx_norm} and \cref{thm:asym_comp_coef} guarantee. 

To test the impact of $S$ on performance, we fixed $\Jup$ = 5,000, and we let $S$ vary between $10^2$ and $10^6$. \Cref{fig:S_vs_error} shows what the results in  \cref{fig:kern_errors} would have looked like had we chosen a different $S$. We clearly see that after around only $S$ = 10,000 there is a phase transition such that increasing S does not further improve performance.  

To better understand if we actually see an exponential compression in $\Jup$ in practice, as our theory suggests, we set $\Jup = 10^5$ (i.e.\ very large) and fixed $S$ = 20,000 as before. We examined the HIGGS dataset consisting of $1.1 \times 10^7$ samples, and let $J$ (the number of compressed features) vary between $500$ and $10^4$. Since GIGA can select the same random feature at different iterations (i.e.\ give a feature higher weight), $J$ reached 8,600 after $10^4$ iterations in \cref{fig:J_impact}. \cref{fig:J_impact} shows that for $J \approx 2 \times 10^3$, increasing J further has negligible impact on kernel approximation performance---only $0.001$ difference in relative error. \cref{fig:J_impact} shows that we are able to compress $\Jup$ by around two orders of magnitude. 

Finally, since our proofs of \cref{thm:approx_norm} and \cref{thm:asym_comp_coef} assume Step 8 of \cref{algfeatcompress} is run using Frank-Wolfe instead of GIGA, we compare in \cref{fig:fw_vs_giga} how the results in \cref{fig:kern_errors} change by using Frank-Wolfe instead. \cref{fig:fw_vs_giga} shows that for $J$ small, GIGA has better approximation quality than FW but for larger $J$, the two perform nearly the same. This behavior agrees with the theory and empirical results of \citet{giga}, where GIGA is motivated specifically for the case of high compression. 

\begin{figure}[!htb]
\begin{centering}
\includegraphics[width=.9\linewidth]{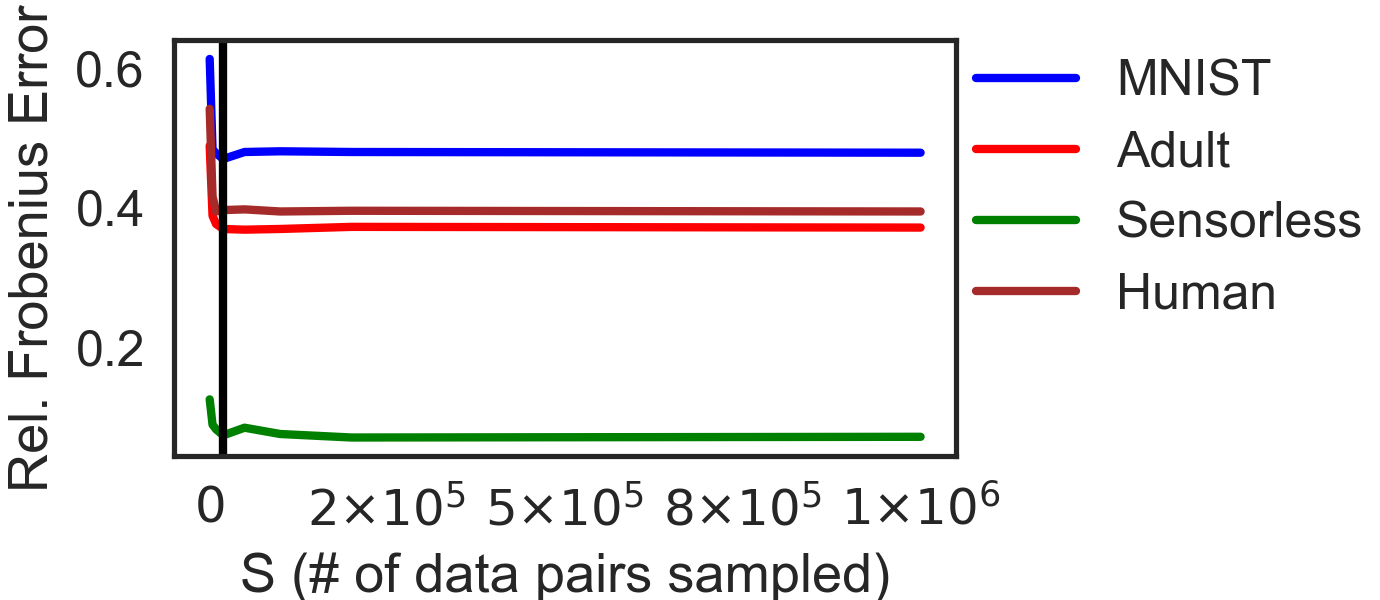}\par
\end{centering}
\caption{We plot the relative Frobenius norm error against $S$ for $\Jup$ fixed at 5,000. The solid black line corresponds to the results found in \cref{fig:kern_errors}.}
\label{fig:S_vs_error}
\end{figure}
\begin{figure}[!htb]
\begin{centering}
\includegraphics[width=.7\linewidth]{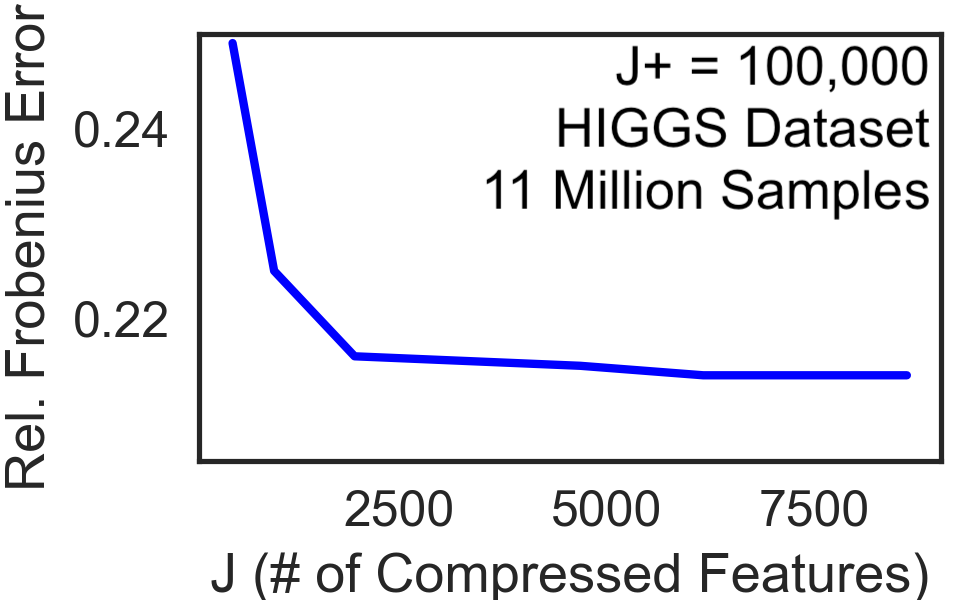}\par
\end{centering}
\caption{Let $S$ = 20,000, $\Jup = 10^5$. We plot the relative Frobenius norm error vs.\ $J$ from $500$ to $10^4$.}
\label{fig:J_impact}
\end{figure}
\begin{figure}[!htb]
\begin{centering}
\includegraphics[width=.7\linewidth]{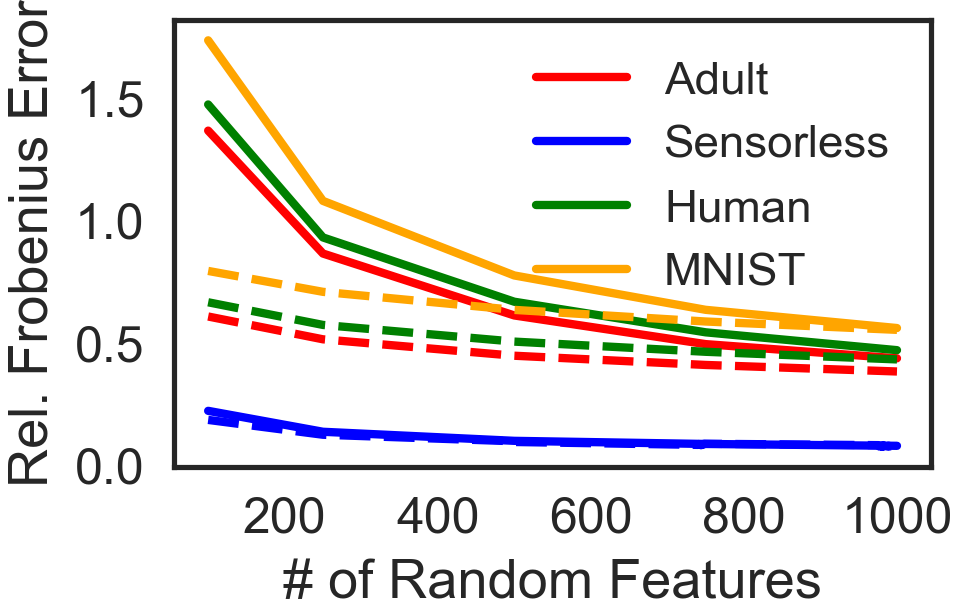}\par
\end{centering}
\caption{The performance of GIGA versus Frank-Wolfe for the experiment described in \cref{fig:kern_errors}. Solid lines correspond to Frank-Wolfe and dashed with GIGA.}
\label{fig:fw_vs_giga}
\end{figure}
\section{Conclusion}
This work presents a new algorithm for scalable kernel matrix approximation.
We first generate a low-rank approximation. We then
find a sparse, weighted subset of the columns of the low-rank factor that minimizes
the Frobenius norm error relative to the original low-rank approximation. Theoretical and empirical
results suggest that our method provides a substantial improvement in scalability
and approximation quality over past techniques. Directions for future 
work include investigating the effects of variance reduction
techniques for the up-projection, using a similar compression technique on features
generated by the Nystr\"om method \citep{nyst_seeg}, and transfer learning of feature weights
for multiple related datasets. 



\subsection*{Acknowledgments}
We thank Justin Solomon for valuable discussions. This research was
supported in part by an ARO YIP Award, ONR (N00014-17-1-2072), an NSF
CAREER Award, the CSAIL-MIT Trustworthy AI Initiative, Amazon, and the MIT-IBM Watson AI Lab.


\appendix
\section{Proof of \cref{thm:asym_comp_coef}} \label{app:proof_asym_comp_coef}
The proofs of \cref{thm:approx_norm} and \cref{thm:asym_comp_coef} rely on the main error bound for the \emph{Hilbert coreset construction problem} given in \cref{eq:fw_obj} \citep{hilb_coresets}. We restate this error bound in \cref{lem:fw_convg}, which  depends on several key quantities given below:  
\begin{itemize}
    \item $c_{ls} \defined \frac{1}{\Jup} \cos(\omega_l^T x_{i_s} + b_l) \cos(\omega_l^T x_{j_s} + b_l)$, such that $1 \leq s \leq S$ and $1 \leq l \leq \Jup$
    \item $\hat{\sigma}_j^2 \defined \frac{1}{S} \sum_{s=1}^S c_{js}^2 = \frac{1}{S} \|R_j\|_2^2$
    \item $\hat{\sigma}^2 \defined \left( \sum_{j=1}^{\Jup} \hat{\sigma}_j \right)^2$
\end{itemize}
\begin{ndefn}{\citep{hilb_coresets}} The \emph{Hilbert construction problem} is based on solving the quadratic program,
\begin{equation} \label{eq:fw_obj}
\argmin_{w \in \R_+^{\Jup}} \quad \frac{1}{S} \| r - r(w)\|_2^2 \quad \text{s.t. $\sum_{j=1}^{\Jup} w_j \hat{\sigma}_j = \hat{\sigma}$}.
\end{equation}
\end{ndefn}

\begin{rmk}
The minimizer of \cref{eq:fw_obj} is $w^* = (1, \cdots, 1)$ since $r(w^*) = r$. However, the goal is to find a sparse $w$. Instead of adding sparsity-inducing constraints (such as $L_1$ penalties), which would lead to computational difficulties for large-scale problems, \citet{hilb_coresets} minimize \cref{eq:fw_obj} greedily through the Frank-Wolfe algorithm. Frank-Wolfe outputs a sparse $w$ since the sparsity of $w$ is bounded by the number of iterations Frank-Wolfe is run for.     
\end{rmk}

\begin{nlem} \label{lem:fw_convg} \cite[Theorem 4.4]{hilb_coresets}
Solving \cref{eq:fw_obj} with $J$ iterations of Frank-Wolfe satisfies
\begin{equation}
\begin{split}
\frac{1}{S} \|r - r(w) \|_2^2 & \leq \frac{\hat{\sigma}^2 \eta^2 \bar{\eta}^2 \nu^2_J}{\bar{\eta}^2 \nu^{-2(J - 2)} + \eta^2(J - 1)} \\
			& \leq \nu^{2J - 2}_J,
\end{split}
\end{equation}
where $0 \leq \nu_J < 1$. Furthermore, $\nu^2_J = 1 - \frac{d^2}{\sigma^2 \bar{\eta}^2}$ where $d$ is the distance from $r$ to the nearest boundary of the convex hull of $\left \{ \frac{\hat{\sigma}}{\hat{\sigma}_j}R_j \right \}_{j=1}^{\Jup}$ and $\bar{\eta}^2 \defined \frac{1}{S} \max_{i,j \in [\Jup]}\left\|\frac{R_i}{\hat{\sigma}_i} - \frac{R_j}{\hat{\sigma}_j} \right\|^2, 0 \leq \bar{\eta} \leq 2$.
\end{nlem}

We prove \cref{thm:asym_comp_coef} first since the main idea is captured in this proof. The proof of \cref{thm:approx_norm} is more involved since we must use a number of concentration bounds to justify subsampling only $S$ datapoint pairs instead of all $\frac{N(N-1)}{2}$ possible datapoint pairs.
Both proofs will also depend on the following constants.
\begin{itemize}
    \item $\sigma_j^2 \defined \frac{1}{V^*} \sum_{s=1}^{V^*} c_{js}^2 = \frac{1}{V^*} \|R_j\|_2^2$
    \item $\sigma^2 \defined \left( \sum_{j=1}^{\Jup} \sigma_j \right)^2$
\end{itemize}
Here, $V^* = \frac{N(N-1)}{2}$, that is when all datapoint pairs above the diagonal are included. $\hat{\sigma}_j^2$ and $\hat{\sigma}^2$ are simply unbiased estimates of $\sigma_j^2$ and $\sigma^2$ based on sampling only $S$ instead of all $V^*$ datapoint pairs.

While \cref{lem:fw_convg} guarantees $0 < \nu_{\Jup} < 1$, it does not guarantee that $\nu_{\Jup} \rightarrow{1}$ as the number of random features $\Jup \rightarrow \infty$. The following Lemma is critical in showing that $\nu_{\Jup}$ does not approach $1$, which would result in no compression.
\begin{nlem} \label{lem:hull_to_spere}
Let $\{x_i\}_{i=1}^K$ be a set of points in $\R^p$ that satisfies \cref{assumptions}(a). Consider the vector $v_{\omega, b} = (\cos(\omega^T x_{i} + b) \cos(\omega^T x_{j} + b))_{i < j, i \in [K-1]} \in \R^{\frac{K(K-1)}{2}}$. Let the unit vector $u_{\omega, b} \coloneqq \frac{v_{\omega, b}}{\| v_{\omega, b} \|}$. If $\omega_j  \overset{\text{i.i.d.}}{\sim} F$ and $b_j \overset{\text{i.i.d.}}{\sim} G$, where $F$ has positive density on all of $\R^p$ and $G$ has positive density on $[0, 2\pi]$, then 
\begin{equation}
\begin{split}
   & \dis\left(\conv \{ u_{\omega_j, b_j} \}_{j=1}^J, \mathcal{S}^{\frac{K(K-1)}{2} - 1}\right) \rightarrow 0  \quad \text{for} \quad J \rightarrow \infty \\
   & \text{s.t.} \ \ \dis(A, B) \coloneqq \max_{a \in A, b \in B} ||a - b||_2.
\end{split}
\end{equation}
Here, $\mathcal{S}^{\frac{K(K-1)}{2} - 1}$ denotes the surface of the unit sphere in $\R^{\frac{K(K-1)}{2}}$.
\end{nlem}
\begin{proof}
By construction, each unit vector $u_i \coloneqq u_{\omega_i, b_i}$ lies on the boundary of the unit sphere in $\R^{\frac{K(K-1)}{2}}$. Hence, $F, G$ induce a distribution on $\mathcal{S}^{\frac{K(K-1)}{2} - 1}$. It suffices to show $\mathcal{S}^{\frac{K(K-1)}{2} - 1}$ has strictly positive density everywhere since, as $J \rightarrow \infty$, any arbitrarily small neighborhood around a collection of points that cover $\mathcal{S}^{\frac{K(K-1)}{2} - 1}$ will be hit by some $u_i$ with probability 1. By standard convexity arguments, the convex hull of the $u_i$ will arbitrarily approach $\mathcal{S}^{\frac{K(K-1)}{2} - 1}$ by taking the radius of the neighborhoods to zero. We now show $\mathcal{S}^{\frac{K(K-1)}{2} - 1}$ has strictly positive density everywhere. Since $u_i$ is the normalized vector of $v_i \coloneqq v_{\omega_i, b_i}$ and each component of $v_i$ is between $-1$ and 1, it suffices to show, by the continuity of the cosine function, that for any $a \in \{-1, 1 \}^{\frac{K(K-1)}{2}}$ there exist some $\omega_i, b_i$ such that $\text{sign}(v_i) \coloneqq (\text{sign}(v_{il}))_{l \in \frac{K(K-1)}{2}}$ equals $a$. Recall that
\begin{equation} \label{eq:cos_form}
    \cos(a)\cos(b) = \frac{1}{2}(\cos(a + b) + \cos(a - b)).
\end{equation}
Take $b_i = 0$. Then, \Cref{eq:cos_form} implies $v_{il} = \frac{1}{2}(\cos(\omega_i^T(x_{i_l} + x_{j_l}) + \cos(\omega^T(x_{i_l} - x_{j_l}))$. Consider the vector $\tilde{v}_i = (\cos(\omega_i^T(x_{i_l} + x_{j_l}), \cos(\omega_i^T(x_{i_l} - x_{j_l}))_{l \in \frac{K(K-1)}{2}}\in \R^{K(K-1)}$. It suffices to show that for any $\tilde{a} \in \{-1, 1\}^{K(K-1)}$, there exists an $\omega_i$ such that $\text{sign}(\tilde{v}_i) = \tilde{a}$. Recall that the cosine function has infinite \emph{VC dimension}, namely that for any labeling $y_1, \cdots, y_M \in \{-1, 1\}$ of distinct points $x_1, \cdots x_M \in \R^p$, there exists an $\omega^*$ such that $\text{sign}(\cos((\omega^*)^T x_m)) = y_m$. Take $M = K(K-1)$, $y_m = \tilde{a}_m$, $x_m = x_{i_m} + x_{j_m}$, and $x_{m+1} = x_{i_m} - x_{j_m}$. Since all the $x_m$ are distinct by \cref{assumptions}(a), we can find an $\omega_i$ such that $\text{sign}(\tilde{v}_i) = \tilde{a}$ as desired.
\end{proof}
We now prove \cref{thm:asym_comp_coef}.
\begin{proof}
Each $R_j \in \R^{\frac{N(N-1)}{2}}$ and the $R_j$'s are i.i.d. since each $\omega_j$ is drawn i.i.d. from $Q$. The induced Hilbert norm $\| \cdot \|_H$ of each $R_j$ is given by $\| R_j \|_H^2 =  \frac{2}{N(N-1)} \| R_j \|^2_2$ \citep{hilb_coresets}. Hence, $\tilde{R_j} \coloneqq \frac{R_j}{\sigma_j}$ is a unit vector in the vector space with norm $\|  \cdot \|_H$. By \cref{lem:hull_to_spere}, 
\begin{equation}
\begin{split}
   & \text{d}\left(\text{ConvexHull}\{\tilde{R}_j \}_{j=1}^{\Jup}, \mathcal{S}^{\frac{N(N-1)}{2} - 1}\right) \rightarrow 0 \\
\end{split}
\end{equation}
Let $\tilde{r} \coloneqq \frac{1}{\sigma} \sum_{j=1}^{\Jup} \sigma_j \tilde{R_j} \in \conv\{\tilde{R}_j \}_{j=1}^{\Jup}$ and observe that $\tilde{r} = \frac{r}{\sigma}$. The distance, which we denote as $d_{\Jup}$, between $\tilde{r}$ and the $\text{ConvexHull}\{  \tilde{R}_j \}_{j=1}^{\Jup}$ approaches $1 - \| \tilde{r} \|_H$ since the $\text{ConvexHull}\{ \tilde{R}_j \}_{j=1}^{\Jup}$ approaches $\mathcal{S}^{\frac{N(N-1)}{2} - 1}$. Hence,
\begin{equation} \label{eq:djplus}
    \lim_{\Jup \rightarrow{\infty}} d_{\Jup} = 1 - \lim_{\Jup \rightarrow{\infty}} \| \tilde{r} \|_H = 1 - \frac{\lim_{\Jup \rightarrow{\infty}} \| r \|_H}{\lim_{\Jup \rightarrow{\infty}} \sigma}. 
\end{equation}
Now,
\begin{equation}
   r_{s} = \frac{1}{\Jup}\sum_{j=1}^{\Jup} c_{js} \overset{\Jup \rightarrow \infty}{\longrightarrow} k(x_{i_s}, x_{j_s}). 
\end{equation}
Hence,  as $\Jup \rightarrow \infty$,
\begin{equation} \label{eq:limit_r}
    \| r \|_H \rightarrow \sqrt{\frac{2}{N(N-1)} \sum_{i < j} (k(x_i, x_j))^2}.
\end{equation}
Now, 
\begin{equation} \label{eq:sigma_limit}
\begin{split}
    \sigma & =  \sum_{j=1}^{\Jup} \sigma_j \\
    & = \sum_{j=1}^{\Jup}  \sqrt{ \frac{1}{V^*} \sum_{s=1}^{V^*} c_{js}^2} \\ 
    & = \sum_{j=1}^{\Jup}  \sqrt{ \frac{1}{V^*} \sum_{s=1}^{V^*}  \frac{1}{\Juppow{2}} \cos^2(\omega_j^T x_{i_s} + b_j) \cos^2(\omega_j^T x_{j_s} + b_j)} \\
    & = \frac{1}{\Jup} \sum_{j=1}^{\Jup}  \sqrt{ \frac{1}{V^*} \sum_{s=1}^{V^*} \cos^2(\omega_j^T x_{i_s} + b_j) \cos^2(\omega_j^T x_{j_s} + b_j)} \\
    & = \sqrt{\frac{2}{N(N-1)}} \frac{1}{\Jup} \sum_{j=1}^{\Jup} \| (\cos(\omega^T_jx_m + b_j) \cos(\omega^T_j x_n + b_j))_{m < n} \|_2 \\
    & \rightarrow \sqrt{\frac{2}{N(N-1)}} \E_{\omega, b} \| (\cos(w^Tx_m + b) \cos(w^Tx_n + b))_{m < n} \|_2
\end{split}
\end{equation}

If $x \neq y$ and $w \neq 0$, then
\begin{equation} \label{eq:strict_ineq}
\begin{split}
    k(x, y) & = \E_{\omega, b} \cos(w^Tx + b) \cos(w^Ty + b) \\ 
    & < \E_{\omega_i, b} |\cos(w^Tx + b) \cos(w^Ty + b)|.
\end{split}
\end{equation}
by Jensen's inequality. Hence,  \cref{eq:strict_ineq} and \cref{assumptions}(a-b) together imply
\begin{equation}
    \frac{\lim_{\Jup \rightarrow{\infty}} \| r \|_2}{\lim_{\Jup \rightarrow{\infty}} \sigma} < 1.
\end{equation}
By \cref{eq:limit_r} and \cref{eq:sigma_limit},
\begin{equation} \label{eq:up_bound_frob}
    \frac{\lim_{\Jup \rightarrow{\infty}} \| r \|_H}{\lim_{\Jup \rightarrow{\infty}} \sigma} \leq \frac{\|K \|_F}{ \E_{\omega, b} \| u(\omega, b) \|_2},
\end{equation}
where $u(\omega, b)$ is defined in \cref{thm:asym_comp_coef}. \Cref{lem:fw_convg} says that $\nu^2_{\Jup} = 1 - \frac{d^2}{\sigma^2 \bar{\eta}^2}$, where $d$ is the distance from $r$ to the nearest boundary of the convex hull of $\left \{ \frac{\sigma}{\sigma_j} R_j \right \}_{j=1}^{\Jup}$. Hence, $d = \sigma d_{\Jup}$ and $\nu^2_{\Jup} = 1 - \frac{ d_{\Jup}^2}{\bar{\eta}^2}$. \cref{eq:djplus} and \cref{eq:up_bound_frob} together imply,
\begin{equation}
    \liminf_{\Jup \rightarrow{\infty}} d_{\Jup} \leq 1 - \frac{\|K \|_F}{ \E_{\omega, b} \| u(\omega, b) \|_2}.
\end{equation}
Therefore, since $0 \leq \bar{\eta}^2 \leq 2$ by \cref{lem:fw_convg}, 
\begin{equation}
\begin{split}
    \limsup_{\Jup \rightarrow{\infty}} \nu_{\Jup}^2 
    &\leq \limsup_{\Jup \rightarrow{\infty}} 1 - \frac{d_{\Jup}^2}{2} \\
    &= 1 - \liminf_{\Jup \rightarrow{\infty}} \frac{d_{\Jup}^2}{2} \\
    &\leq 1 - \frac{\left(1 - \frac{\|K \|_F}{ \E_{\omega, b} \| u(\omega, b) \|_2}\right)^2}{2}. 
\end{split}
\end{equation}
\end{proof}

\section{Proof of \cref{thm:approx_norm}} \label{A:proof_approx_inner}

The following technical lemma is needed to derive the probability bound in \cref{thm:approx_norm}.
\begin{nlem} \label{lem:ratio_bound}
Suppose $\frac{\sigma^2}{\Juppow{2}\sigma_i^2} \leq M$ for some $1 \leq M < \infty$ for all $i \in [\Jup]$. For $S \geq 8 \frac{M^2}{\sigma^4} \log \left ( \frac{2\Jup}{\delta^2} \right)$ 
\begin{equation}
\pr{\frac{\hat{\sigma}^2}{\Juppow{2}\hat{\sigma_i}^2} \geq  5M } \leq \delta
\end{equation}
for all $i  \in [\Jup]$.
\end{nlem}
\begin{proof}
Notice that 
\begin{equation*}
\begin{split}
\E_{i_s, j_s} \hat{\sigma_l}^2 &= \frac{1}{S} \sum_{s=1}^S \E_{i_s, j_s} c_{ls}^2 \\
						&= \frac{1}{N^2} \sum_{s=1}^{N^2} c_{ls}^2 \\
                        & = \sigma_l^2.
\end{split}
\end{equation*}
Hence, $\hat{\sigma_l}^2$ is an unbiased estimator of $\sigma_l^2$. Each $c_{ls}^2 \leq \frac{1}{\Jup^2}$ is a bounded random variable, and the collection of random variables $\{c_{ls}^2\}_{s=1}^S$ are i.i.d.\ since $i_s, j_s \overset{\text{i.i.d.}}{\sim} \pi$. Hence, by Hoeffding's inequality,
\begin{equation} \label{eq:sig_hat_i_bound}
\pr{|\hat{\sigma}^2_l - \sigma^2_l| \geq t } \leq 2\exp \left ( -2 S \Juppow{4} t^2 \right ).
\end{equation}
Define the event $A_t \defined \cup_{i=1}^{\Jup} \{ |\hat{\sigma}^2_i - \sigma^2_i| < t  \}$ and pick $t$ such that $t \leq \min_{i \in [\Jup]} \sigma_i^2 $. Since $\sigma_i^2 \geq \frac{\sigma^2}{M}$ by assumption, it suffices to pick $0 < t \leq \frac{\sigma^2}{M}$. Conditioned on $A_t$, $\hat{\sigma}_i \leq \sqrt{\sigma_i^2 + t} \leq \sigma_i + \sqrt{t}$, which implies $\hat{\sigma}^2 \leq (\sigma + \Jup \sqrt{t})^2$. Therefore, 
\begin{equation}
\begin{split} \label{eq:bound_on_ratio}
\pr{\frac{\hat{\sigma}^2}{\Juppow{2}\hat{\sigma_i}^2} \geq cM } & = \pr{A_t^c \cup \left \{ \frac{\hat{\sigma}^2}{\Juppow{2}\hat{\sigma_i}^2} \geq cM \right \} } +  \pr{A_t \cup \left \{ \frac{\hat{\sigma}^2}{\Juppow{2}\hat{\sigma_i}^2} \geq cM \right \} } \\
		& \leq \pr{A_t^c} + \pr{A_t, \left \{ \frac{\hat{\sigma}^2}{\Juppow{2}\hat{\sigma_i}^2} \geq cM \right \}} \\
        & \leq \pr{A_t^c} + \pr{\frac{\hat{\sigma}^2}{\Juppow{2}\hat{\sigma_i}^2} \geq cM \mid A_t} \\
        & \leq \pr{A_t^c} + \pr{\frac{(\sigma + \sqrt{t}\Jup)^2}{\Juppow{2}(\sigma_i^2 - t)} \geq cM \mid A_t}.
\end{split}
\end{equation}
Notice that $\pr{\frac{(\sigma + \sqrt{t}\Jup)^2}{\sigma_i^2 - t} \geq cM^2 \mid A_t}$ is either 0 or 1 since $\sigma_i$ and $\sigma$ are constants. We pick $t$ so that this probability is 0. To pick $t$, notice that,
\begin{equation}
\begin{split}
\frac{(\sigma + \sqrt{t}\Jup)^2}{\Juppow{2}(\sigma_i^2 - t)} & = \frac{\left ( \frac{\sigma}{\sigma_i} + \frac{\sqrt{t}\Jup}{\sigma_i} \right)^2}{\Juppow{2}(1 - \frac{t}{\sigma_i^2})} \\
		& \leq \frac{\left( \Jup\sqrt{M} + \frac{\Jup \sqrt{tM}\Jup}{\sigma} \right)^2}{\Juppow{2}(1 - \frac{t}{\sigma_i^2})} \\
		& \leq \frac{M \left (1  + \frac{\sqrt{t}\Jup}{\sigma} \right)^2}{1 - \frac{M \Juppow{2} t}{\sigma^2}},
\end{split}
\end{equation}
where the last inequality holds as long as $0 < t < \frac{\sigma^2}{M\Juppow{2}}$ and follows by noting that $\frac{1}{\sigma_i^2} \leq \frac{M\Juppow{2}}{\sigma^2}$ by assumption. Pick $t = \frac{\sigma^2}{4\Juppow{2} M}$. Since $0 \leq \sigma \leq 1$, this choice of $t$ implies $\frac{M \left( 1 + \frac{\sqrt{t}\Jup}{\sigma}\right)^2}{1 - \frac{M\Juppow{2} t}{\sigma^2}} \leq 5M$. Hence, for $c = 5$ and this choice of $t$, $\pr{\frac{(\sigma + \sqrt{t}\Jup)^2}{\Juppow{2}(\sigma_i^2 - t)} \geq 5M \mid A_t} = 0$. Combining \cref{eq:bound_on_ratio} and \cref{eq:sig_hat_i_bound}, we have by a union bound that,
\begin{equation}
\pr{\frac{\hat{\sigma}^2}{\Juppow{2}\hat{\sigma_i}^2} \geq 5M } \leq 2\Jup \exp \left ( -\frac{1}{8}S \frac{\sigma^4}{M^2} \right ),
\end{equation}
for all $i \in [\Jup]$. Solving for $S$ by setting the right hand side above to $\delta$ yields the claim.
\end{proof}
We have all the pieces to prove \cref{thm:approx_norm}. We follow the proof strategy in \cite[Theorem 5.2]{hilb_coresets}.
\begin{proof} 
Let $R^* = \left [ \zup_{1}^T \circ \zup_{1}^T, \cdots \zup_{{N-1}}^T \circ \zup_{N}^T, \zup_{N}^T \circ \zup_{N}^T \right] \in \R^{\Jup \times N^2}$. Notice,
\begin{equation} \label{eq:exact_norm}
 \frac{1}{N^2}\|\Zup\Zup^T - Z(w)Z(w)^T\|^2_F = (1 - w)^T \frac{R^*}{N} \frac{R^{*T}}{N}(1 - w).
\end{equation}
We approximate \cref{eq:exact_norm} with $(1 - w)^T \frac{R}{\sqrt{S}} \frac{R^{T}}{\sqrt{S}}(1 - w)$ and bound the error. Suppose
\begin{equation*}
    D^* \defined \max_{i,j \in [\Jup]} \left| \left (\frac{R^*}{N} \frac{R^{*T}}{N} \right)_{ij} - \left (\frac{R}{\sqrt{S}} \frac{R^{T}}{\sqrt{S}} \right)_{ij} \right| \leq \frac{\err}{2}.
\end{equation*}
Then,
\begin{equation}
\begin{split}
(1 - w)^T \frac{R^*}{N} \frac{R^{*T}}{N} (1 - w) - (1 - w)^T \frac{R}{\sqrt{S}} \frac{R^{T}}{\sqrt{S}} (1 - w) & \leq \sum_{i, j \in [\Jup]} |w_i - 1| |w_j - 1| D^* \\
			& \leq \| w - 1 \|_1^2 \frac{\err}{2}.
\end{split}
\end{equation}
Notice,
\begin{equation}
\begin{split}
\E_{i_s, j_s} \left [ \left (\frac{R}{\sqrt{S}} \frac{R^{T}}{\sqrt{S}} \right)_{ij} \right] & = \E_{i_s, j_s} \left[ \frac{1}{S} \sum_{s=1}^S c_{is}c_{js} \right ] \\
					& = \frac{1}{S} \sum_{s=1}^S E_{i_s, j_s} \left[ c_{is}c_{js} \right ] \\
                    & = E_{i_s, j_s} \left[ c_{is}c_{js} \right ] \\
                    & =  \frac{1}{N^2} \sum_{s=1}^{N^2} c_{is}c_{js} \\
                    & = \left (\frac{R^*}{N} \frac{R^{*T}}{N} \right)_{ij}.
\end{split}
\end{equation}
Hence, the i.i.d.\ collection of random variables $\{c_{is}c_{js}\}_{s=1}^S$ yields an unbiased estimate of $\left (\frac{R^*}{N} \frac{R^{*T}}{N} \right)_{ij}$. Each $c_{is}c_{js}$ is bounded by $\frac{1}{\Juppow{2}}$. Therefore, by Hoeffding's inequality and a simple union bound,   
\begin{equation}
\pr{ D^* \geq \frac{\err}{2}} \leq 2 \Juppow{2} \exp{\left(-2S\Juppow{4} {\err}^2 \right)}.
\end{equation}
Setting the right-hand side to $\frac{\delta^{*}}{2}$ and solving for $\frac{\err}{2}$ implies with probability at least $1 - \frac{\delta^*}{2}$, 
\begin{equation}
\frac{\err}{2}  \leq \frac{1}{\sqrt{S}\Jup^2} \log \left[\frac{4\Juppow{2}}{\delta^*} \right]^{\frac{1}{2}}. 
\end{equation}
Hence, with probability at least $1 - \frac{\delta^{*}}{2}$, 
\begin{equation*}
\begin{split}
\frac{1}{N^2}\|\Zup\Zup^T - Z(w)Z(w)^T\|^2_F & \leq (1 - w)^T \frac{R}{\sqrt{S}} \frac{R^{T}}{\sqrt{S}} (1 - w) + \| 1 - w \|_1^2 \frac{1}{\sqrt{S}\Juppow{2}} \log \left[\frac{4\Juppow{2}}{\delta^*} \right]^{\frac{1}{2}} \\
					& = \frac{1}{S} \|r - r(w) \|_2^2 + \| 1 - w \|_1^2 \frac{1}{\sqrt{S}\Juppow{2}} \log \left[\frac{4\Juppow{2}}{\delta^*} \right]^{\frac{1}{2}}
\end{split}
\end{equation*}
\Cref{lem:fw_convg} implies that there exists a $0 \leq \nu < 1$ such that $\frac{1}{S} \|r - r(w) \|_2^2 \leq \nu^{2J - 2}$. Since $\nu$ depends on the pairs $i_l, j_l$ picked, we can take $\nu^*$ to be the largest $\nu$ possible. Since the set of all possible $S$ pairs is finite, that implies $0 \leq \nu^* < 1$. Hence, setting $J = \frac{1}{2}\log_{\nu^*}\left (\frac{\err}{2} \right ) + 2$ guarantees that $\frac{1}{S} \|r - r(w) \|_2^2 \leq \frac{\err}{2}$ for any collection of drawn $i_l, j_l, 1 \leq l \leq S$. Assume for any $a \in (0, 1]$ and $\delta > 0$, we can find an $M$ such that

\begin{equation} \label{eq:M_delta}
    \pr{\max_j \nicefrac{\sigma^2}{(\Juppow{2} \sigma_j^2)} > M} < a\delta.
\end{equation}

If \cref{eq:M_delta} holds, we may assume $\max_j \nicefrac{\sigma^2}{(\Juppow{2} \sigma_j^2)} < M$ by setting $M$ large enough since we just need a $1- \delta$ probabilistic guarantee. By the polytope constraint in \cref{eq:fw_obj}, $w^{*}_i \leq \frac{\hat{\sigma}}{\hat{\sigma}_i}$ for all $i \in [\Jup]$. Without loss of generality, assume the first $J$ components of $w^*$ can be the only non-zero values since $w^*$ is at least $J$ sparse. For $S \geq 8 \frac{M^4}{\sigma^4} \log \left ( \frac{2\Jup}{\delta^2} \right)$, \cref{lem:ratio_bound} implies with probability at least $1 - \frac{\delta^{*}}{2}$,  
\begin{equation}
\begin{split}
\| 1 - w^{*} \|_1^2 & \leq \left ( \frac{\hat{\sigma}}{\hat{\sigma}_i}J + (\Jup - J) \right )^2 \\
		& \leq \left ( JM \Jup + \Jup) \right )^2 \\
        & \leq (2 JM\sqrt{5} \Jup)^2 \\
        & \leq 10 \Juppow{2} M^2 J^2 \\
        & \leq 10 \Juppow{2} M^2 \frac{(\log \frac{2}{\err})^2}{(\log \nu)^2}
\end{split}
\end{equation}
Therefore, with probability at least $1 - \delta^{*}$, 
\begin{equation}
\frac{1}{N^2}\|\Zup\Zup^T - Z(w)Z(w)^T\|^2_F \leq \frac{\err}{2} + \frac{10M^2(\log \frac{2}{\err})^2}{\sqrt{S}(\log \nu)^2} \log \left[\frac{4\Juppow{2}}{\delta^*} \right]^{\frac{1}{2}}.  
\end{equation}
Finally, setting $S \geq \max \left( \frac{100}{{\err}^2} \left[ M\frac{(\log \frac{2}{\err})}{(\log \nu)} \right]^4 \log \left[\frac{4\Jup^2}{\delta^*} \right] ,8 \frac{M^4}{\sigma^4} \log \left ( \frac{2\Jup}{\delta^2} \right) \right)$ implies $\frac{1}{N^2}\|\Zup\Zup^T - Z(w)Z(w)^T\|^2_F \leq \err$ with probability at least $1 - \delta^{*}$ which matches the rate provided in \cref{thm:approx_norm}. It remains to show \cref{eq:M_delta}. Notice that 
\begin{equation}
    \frac{\sigma}{\Jup \sigma_j} = \frac{1}{\Jup} + \frac{1}{\Jup} \sum_{i \neq j} \tilde{\sigma_{ij}},
\end{equation}
where $\sigma_{ij} \coloneqq \frac{\sigma_i}{\sigma_j}$. Notice that each $\sigma_{ij}$ are i.i.d. for $i \neq j$. Let the $\mu_j = \E \sigma_{ij}$ and $s_j$ be the standard deviation of $\sigma_{ij}$. Since each $\sigma_j$ is i.i.d. that implies $\mu_j$ and $s_j$ are both constant across $j$ so we drop the subscript. By a union bound, it suffices to show for any $\tau > 0$ we can find an $M$ such that 
\begin{equation} \label{eq:m_reduction}
    \pr{\max_{1 \leq j \leq \Jup} \frac{1}{\Jup} \sum_{i \neq j} \tilde{\sigma_{ij}} > M} < \tau.  
\end{equation}
By Chebyshev's inequality, 
\begin{equation} \label{eq:cheb_bound}
    \pr{\frac{1}{\Jup} \sum_{i \neq j} \tilde{\sigma_{ij}} - \mu > \frac{cs}{\Jup}} \leq \frac{1}{c^2}.
\end{equation}
Take $c=\Jup \tau$. Then,
\begin{equation}
    \pr{\frac{1}{\Jup} \sum_{i \neq j} \tilde{\sigma_{ij}} - \mu > \frac{cs}{\Jup}} \leq \frac{1}{\Juppow{2} \tau} < \tau.
\end{equation}
By a union bound, \cref{eq:cheb_bound} implies $$\pr{\max_{1 \leq j \leq \Jup} \frac{1}{\Jup} \sum_{i \neq j} \tilde{\sigma_{ij}} > M} < \frac{1}{\tau \Jup} < \tau$$ for $M = \mu + s\tau$ as desired. 

The proof showing that  $\limsup_{\Jup \rightarrow \infty} \nu_{\Jup} < 1$ is the same as the proof \cref{thm:asym_comp_coef}.

\end{proof}

\section{Runtime analysis of methods} \label{A:runtime_analysis}

The ridge regression and PCA runtimes depend on the number of features used, as specified in \cref{tab:impact_frob_norm}, and therefore follow from the first column of the table. 

First, we show that using RFM with $\Jup = O \left( \frac{1}{\err}\log \frac{1}{\err} \right)$ number of random features ensures that $\frac{1}{N^2}\|K - \hat{K} \|_F^2 = O(\err)$ with high probability. By a union bound, $\pr{\frac{1}{N^2}\|K - \hat{K} \|_F^2 \leq \err} \geq \pr{\max_{i, j \in [N]} | K_{ij} - \hat{K}_{ij} | \leq \sqrt{\err}}$. Now, Claim 1 of \cite{rahimi_rf} implies 
\begin{equation} \label{eq:rahami_bound}
\pr{\max_{i, j \in [N]} | K_{ij} - \hat{K}_{ij} | \geq \sqrt{\err}} = \bigO{\frac{1}{\err}e^{-\Jup \err}}.
\end{equation}
Setting the right-hand side of \cref{eq:rahami_bound} to some fixed probability threshold $\delta^{*}$ implies $\Jup = \bigO{\frac{1}{\err}\log \left(\frac{1}{\err \delta^*}\right)}$. Since $\delta^{*}$ is some fixed constant, $\Jup = \bigO{\frac{1}{\err}\log \frac{1}{\err}}$ number of random features suffices for an $O(\epsilon)$ error guarantee. Hence, it suffices to use $\Jup = O \left( \frac{1}{\err}\log \frac{1}{\err} \right)$ as the up-projection dimension for both RFM-FW and RFM-JL. 

To prove the bounds for RFM-FW, take $S = \Omega(\Juppow{2}(\log \Jup)^2)$. It is straightforward to check that this choice of $S$ satisfies the requirements of \cref{thm:approx_norm}. By \cref{thm:approx_norm}, it suffices to set $J = \bigO{\log \Jup}$ for an $O(\epsilon)$ error guarantee.  Hence, \Cref{algfeatcompress} takes $O(S\Jup \log \Jup)$ time to compute the random feature weights $w$ since Frank-Wolfe has to be run for a total of $O(\log \Jup)$ iterations. Finally, it takes $O(N\log \Jup)$ to apply these $O(\log \Jup)$ weighted random features to the $N$ datapoints. We conclude by proving the time complexity of RFM-JL.  

Denote $\tilde{x_i} \coloneqq (Z_+)_i \in \R^{\Jup}$ as the mapped datapoints from RFM.  Let $A \in \R^{J \times \Jup}$ for $J \leq \Jup$ be a matrix filled with i.i.d.\ $N(0, \frac{1}{J})$ random variables for the JL compression step. Let $f(x) \defined Ax$. It suffices to pick a $J$ such that,
\begin{equation} \label{eq:max_inner_prod}
\pr{\max_{i, j \in [N]} \left | \tilde{x_i}^T\tilde{x_j} - f(\tilde{x_i})^T f(\tilde{x_j})  \right | \geq \sqrt{\err}} \leq  \delta^*
\end{equation}
for RFM-JL. We use the following corollary from \citet[Corollary 2.1]{Jl_inner_note} to bound the above probability.

\begin{nlem} \label{lem:JL_inner_prod}
Let $u, v \in \R^d$ and such that $\|u\| \leq 1$ and $\|v\| \leq 1$. Let $f(x) = Ax$, where $A$ is a $k \times d, k \leq d$ matrix of i.i.d.\ $N(0, \frac{1}{k})$ random variables. Then, 
\begin{equation}
\pr{\mid u^Tv - f(u)^Tf(v) \mid} \leq 4e^{-\frac{1}{4}(\epsilon^2 - \epsilon^3)k}.
\end{equation}
\end{nlem}
$\|\tilde{x_i}\|_2 = 1$ since $\tilde{x_i} = \frac{1}{\sqrt{\Jup}}\left( \cos(\omega_1^T x_i + b), \cdots, \cos(\omega_{\Jup}^T x_i + b) \right)$. Hence, we may apply \cref{lem:JL_inner_prod} to $\tilde{x_i}$. By a union bound and an application of \cref{lem:JL_inner_prod}, \cref{eq:max_inner_prod} is bounded by $\bigO{N^2 e^{-J\err}}$. Setting $N^2 e^{-J\err}$ equal to $\delta^*$ and solving for $J$ implies that $J = \Omega\left( \frac{1}{\err} \log \left( \frac{N^2}{\delta^*} \right)\right)$. Hence, $J = \bigO{ \frac{1}{\err} \log N}$. Now, $\bigO{\frac{1}{\err}} = \bigO{\frac{\Jup}{\log \frac{1}{\err}}}$ which implies $J = \bigO{ \frac{\Jup \log N}{\log \frac{1}{\err}}}$. Since $N > \Jup > \bigO{\frac{1}{\err}}$, $J = \Omega(\Jup)$ suffices for an for an $O(\epsilon)$ error guarantee.  While the JL algorithm typically takes $\bigO{N\Jup k}$ time to map a $N \times \Jup$ matrix to a $N \times k$ matrix, the techniques in \citet[Section 3.5]{compact_rf} show that only $\bigO{N\Jup \log J}$ time is required by using the Fast-JL algorithm.    

\section{Impact of kernel approximation} \label{A:downstream}
Here we provide the precise error bound and runtimes for kernel ridge regression, kernel SVM, and kernel PCA when using a low-rank factorization $ZZ^T$ of $K$. We denote $X \subset \R^p$ as the input space and define $c > 0$ such that $K(x, x) \leq c$ and $\hat{K}(x, x) \leq c$ for all $x \in X$. This condition is verified with $c = 1$ for Gaussian kernels for example. All the bounds provided follow from \cite{cortes_kern_approx,thesis}, where we simply replace the spectral norm with the Frobenius norm since the Frobenius norm upper bounds the spectral norm.
\subsection{Kernel ridge regression}
Exact kernel ridge regression takes $O({N^3})$ since $K$ must be inverted. Suppose $K \approx ZZ^T \defined \hat{K}$, where $Z$ could be found using RFM for example. Running ridge regression with the feature matrix $Z$ just requires computing and inverting the covariance matrix $Z^T Z \in R^{J \times J}$ which takes $\Theta({\max(J^3, NJ^2)})$ time. \Cref{prop:ridge_error} quantifies the error between the regressor obtained from $K$ and the one from $\hat{K}$.

\begin{nprop} \label{prop:ridge_error}
(Proposition 1 of \cite{cortes_kern_approx}) Let $\hat{f}$ denote the regression function returned by kernel ridge regression when using the approximate kernel matrix $\hat{K} \in \mathbb{R}^{N\times M}$, and $f^*$ the function returned when using the exact kernel matrix $K$. Assume that every response $y$ is bounded in absolute value by $M$ for some $0 < M < \infty$. Let $\lambda \defined N\lambda_0 > 0$ be the ridge parameter. Then, the following inequality holds for all $x \in X$:
\begin{equation*}
\begin{split}
|\hat{f}(x) - f^*(x)| & \leq \frac{c M}{\lambda_0^2 N}\| \hat{K} - K\|_2  \\
				& \leq \frac{c M}{\lambda_0^2 N}\| \hat{K} - K\|_F \\
                & = \bigO{\frac{1}{N} \| \hat{K} - K\|_F}
\end{split}
\end{equation*}
\end{nprop}

\subsection{Kernel SVM}
Kernel SVM regression takes $O({N^3})$ using $K$ since $K$ must be inverted. Again suppose $K \approx ZZ^T \defined \hat{K}$. Then, training a linear SVM via dual-coordinate decent on $Z$ has time complexity $\bigO{NJ\log \rho}$, where $\rho$ is the optimization tolerance \cite{lin_svm_train}.

\begin{nprop} \label{prop:svm_error}
(Proposition 2 of \cite{cortes_kern_approx}) Let $\hat{f}$ denote the hypothesis returned by SVM when using the approximate kernel matrix $\hat{K}$, $f^*$ the hypothesis returned when using the exact kernel matrix $K$, and $C_0$ be the penalty for SVM. Then, the following inequality holds for all $x \in X$:
\begin{equation*}
\begin{split}
|\hat{f}(x) - f^*(x)| & \leq \sqrt{2} c^{\frac{3}{4}} C_0 \| \hat{K} - K\|_2^{\frac{1}{4}} \left [1 + \frac{\| \hat{K} - K\|_2^{\frac{1}{4}}}{4c} \right] \\
					 & \leq \sqrt{2} c^{\frac{3}{4}} C_0 \| \hat{K} - K\|_F^{\frac{1}{4}} \left [1 + \frac{\| \hat{K} - K\|_F^{\frac{1}{4}}}{4c} \right]. \\
                     & = \bigO{\| \hat{K} - K\|_F^{\frac{1}{2}}}.
\end{split}
\end{equation*}
\end{nprop}
\subsection{Kernel PCA}
We follow \cite{thesis} to understand the effect matrix approximation has on kernel PCA. For a more in-depth analysis, see pg. 92-98 of \cite{thesis}. Without loss of generality, we assume the data are mean zero.   

Let $\Phi(\cdot)$ be the unique feature map such that $k(x, y) = \inprod{\Phi(x)}{\Phi(y)}$. Let the feature covariance matrix be denoted as $\Sigma_{\Phi} \coloneqq \Phi(X_N) \Phi(X_N)^T$, where $\Phi(X_N) \defined \left [\Phi(x_1) \cdots \Phi(x_n) \right]$. Since the rank of $\Sigma_{\Phi}$ is at most $N$, let $v_i$ $1 \leq i \leq N$ be the $N$ singular vectors of $\Sigma_{\Phi}$. For certain kernels, e.g., the RBF kernel, the $v_i$ are infinite dimensional. However, the projection of $\Phi(x)$ onto each $v_i$ is tractable to compute via the kernel trick: 
\begin{equation}
\Phi(x)^T v_i = \Phi(x) \frac{\Phi(X_N)u_i}{\sqrt{\sigma_i}} = \frac{k_x^T u_i}{\sqrt{\sigma_i}},
\end{equation}
where $k_x \defined (K(x_1, x), \cdots, K(x_N, x))$ and $u_i$ is the ith singular vector of $K$ with associated eigenvalue $\sigma_i$. Often, the goal is to project $\Phi(x)$ onto the first $l$ eigenvectors of $\Sigma_{\Phi}$ for dimensionality reduction. To analyze the error of the projection, let $P_{V_l}$ be defined as the subspace $V_l$ spanned by the top $l$ eigenvectors of $\Sigma_{\Phi}$. Then, the \emph{average empirical residual} $R_l(K)$ of a kernel matrix $K$ is defined as,
\begin{equation}
\begin{split}
R_l(K) & \defined \frac{1}{N}\sum_{n=1}^N \| \Phi(x_n) \|^2 - \frac{1}{N}\sum_{n=1}^N \|P_{V_l}(\Phi(x_n))\|^2 \\
			& = \sum_{i > l} \sigma_i
\end{split}
\end{equation}
$R_l(K)$ is simply the spectral error of a low-rank decomposition of $\Sigma_{\Phi}$ using the SVD. If we instead use $\hat{K}$ for the eigendecomposition, the following proposition bounds the difference between $R_l(K)$ and $R_l(\hat{K})$.

\begin{nprop} (Proposition 5.4 of \cite{thesis}) For $R_l(K)$ and $R_l(\hat{K})$ defined as above,
\begin{equation*}
\begin{split}
|R_l(K) - R_l(\hat{K})| & \leq \left(1 - \frac{l}{N}\right) \|K - \hat{K} \|_2 \\
						& \leq \left(1 - \frac{l}{N}\right) \|K - \hat{K} \|_F.
\end{split}
\end{equation*}
\end{nprop}

\section{Additional Experiments} \label{A:add_experiments}
As stated in \cref{sec:experiments}, our method may be applied on top of other random feature methods. In particular, many previous works have reduced the number of random features needed for a given level of approximation by sampling them from a different distribution (e.g., through importance sampling or Quasi-Monte-Carlo techniques). Regardless of the way the random features are sampled, our method can still be used for compression.

\begin{figure}
\begin{centering}
\includegraphics[width=.8\linewidth]{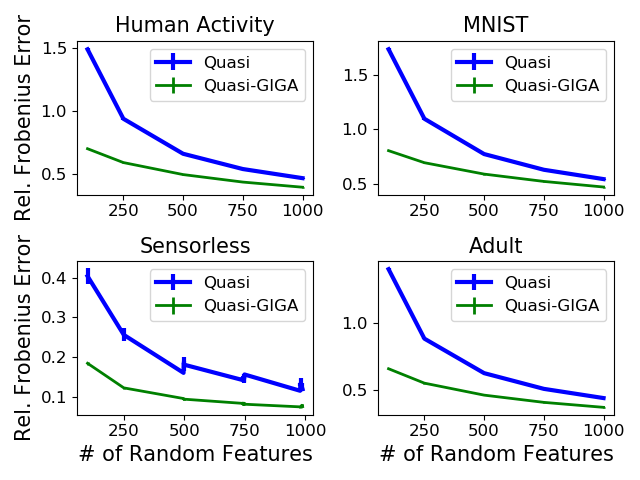}\par
\end{centering}
\vspace{-.1in}
\caption{Kernel matrix approximation errors. Lower is better. Each point denotes the average over 20 simulations and the error bars represent one standard deviation. The HALTON sequence was used to generate the quasi random features.}
\label{fig:kern_errors_quasi}
\vspace{-.1in}
\end{figure}

\begin{figure}
\begin{centering}
\includegraphics[width=.8\linewidth]{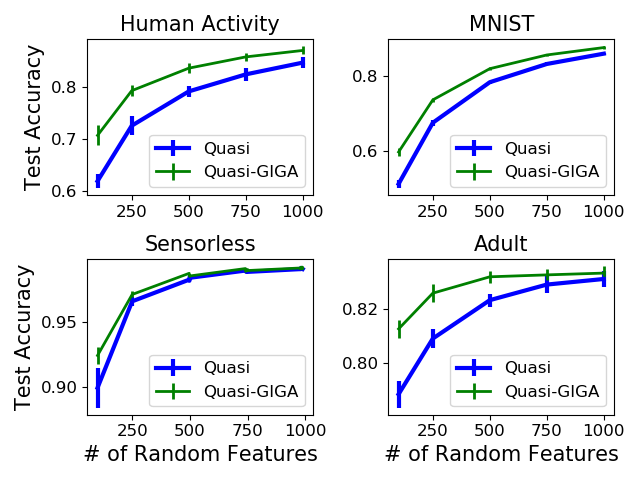}\par
\end{centering}
\vspace{-.1in}
\caption{Classification accuracy. Higher is better. Each point denotes the average over 20 simulations and the error bars represent one standard deviation. The HALTON sequence was used to generate the Quasi random features.}
\label{fig:class_acc_quasi}
\vspace{-.1in}
\end{figure}

To demonstrate this point further, we consider generating random features using Quasi-Monte-Carlo \citep{quasi_rf}. Quasi random features work by generating a sequence of points from a (low-discrepancy) grid of points  in $[0, 1]^p$. Points are sampled from the target random-features distribution $Q$ by applying the inverse CDF of $Q$ on each of these points in the sequence. In \citet{quasi_rf}, the authors showed that generating random features in this way improved performance over the classical random features method provided in \citet{rahimi_rf}.  In \cref{fig:kern_errors_quasi} and \cref{fig:class_acc_quasi}, we see that our method is able to compress the number of quasi random features, which is similar to the behavior in \cref{fig:kern_errors} and \cref{fig:class_acc}. Note that the experimental setup is exactly the same as in \cref{sec:experiments} except that the random features are now generated using Quasi-Monte-Carlo.

\clearpage

\bibliographystyle{abbrvnat} 
\bibliography{references,jhh-references}

\begin{thebibliography}{54}
\providecommand{\natexlab}[1]{#1}
\providecommand{\url}[1]{\texttt{#1}}
\expandafter\ifx\csname urlstyle\endcsname\relax
  \providecommand{\doi}[1]{doi: #1}\else
  \providecommand{\doi}{doi: \begingroup \urlstyle{rm}\Url}\fi

\bibitem[Avron et~al.(2016)Avron, Sindhwani, Yang, and Mahoney]{quasi_rf}
H.~Avron, V.~Sindhwani, J.~Yang, and M.~W. Mahoney.
\newblock {Quasi-Monte Carlo} feature maps for shift-invariant kernels.
\newblock \emph{Journal of Machine Learning Research}, pages 1--38, 2016.

\bibitem[Avron et~al.(2017)Avron, Kapralov, Musco, Musco, Velingker, and
  Zandieh]{Musco}
H.~Avron, M.~Kapralov, C.~Musco, C.~Musco, A.~Velingker, and A.~Zandieh.
\newblock Random {Fourier} features for kernel ridge regression: Approximation
  bounds and statistical guarantees.
\newblock In \emph{International Conference on Machine Learning}, 2017.

\bibitem[Balcan et~al.(2008)Balcan, Blum, and Vempala]{Balcan:2008}
M.~Balcan, A.~Blum, and S.~Vempala.
\newblock On kernels, margins, and low-dimensional mappings.
\newblock In \emph{Algorithmic Learning Theory}, pages 1--12, 2008.

\bibitem[Boser et~al.(1992)Boser, Guyon, and Vapnik]{large_margin}
B.~Boser, I.~Guyon, and V.~Vapnik.
\newblock A training algorithm for optimal margin classifiers.
\newblock In \emph{Workshop on Computational Learning Theory}, pages 144--152,
  1992.

\bibitem[Campbell and Broderick(2018)]{giga}
T.~Campbell and T.~Broderick.
\newblock {Bayesian} coreset construction via greedy iterative geodesic ascent.
\newblock In \emph{International Conference on Machine Learning}, 2018.

\bibitem[Campbell and Broderick(2019)]{hilb_coresets}
T.~Campbell and T.~Broderick.
\newblock Automated scalable {Bayesian} inference via {Hilbert} coresets.
\newblock \emph{Journal of Machine Learning Research}, 2019.

\bibitem[Candes and Tao(2007)]{dantzig}
E.~Candes and T.~Tao.
\newblock The {Dantzig} selector: Statistical estimation when p is much larger
  than n.
\newblock \emph{The Annals of Statistics}, pages 2313--2351, 2007.

\bibitem[Chang et~al.(2017)Chang, Li, Yang, and Poczos]{stein_rf}
W.~Chang, C.~Li, Y.~Yang, and B.~Poczos.
\newblock Data-driven random {Fourier} features using {Stein} effect.
\newblock In \emph{International Joint Conference on Artificial Intelligence},
  pages 1497--1503, 2017.

\bibitem[Chen et~al.(1998)Chen, Donoho, and Saunders]{atom_pursuit}
S.~Chen, D.~Donoho, and M.~Saunders.
\newblock Atomic decomposition by basis pursuit.
\newblock \emph{SIAM Journal on Scientific Computing}, pages 33--61, 1998.

\bibitem[Chwialkowski et~al.(2016)Chwialkowski, Strathmann, and
  Gretton]{Chwialkowski:2016}
K.~Chwialkowski, H.~Strathmann, and A.~Gretton.
\newblock A kernel test of goodness of fit.
\newblock In \emph{International Conference on Machine Learning}, 2016.

\bibitem[Cortes et~al.(2010)Cortes, Mohri, and Talwalkar]{cortes_kern_approx}
C.~Cortes, M.~Mohri, and A.~Talwalkar.
\newblock On the impact of kernel approximation on learning accuracy.
\newblock In \emph{International Conference on Artificial Intelligence and
  Statistics}, 2010.

\bibitem[Daniely et~al.(2017)Daniely, Frostig, Gupta, and
  Singer]{compo_kernels}
A.~Daniely, R.~Frostig, V.~Gupta, and Y.~Singer.
\newblock Random features for compositional kernels.
\newblock \emph{arXiv:1703.07872}, 2017.

\bibitem[Drineas and Mahoney(2005)]{nystroem}
P.~Drineas and M.~Mahoney.
\newblock On the {Nystr\"{o}m} method for approximating a gram matrix for
  improved kernel-based learning.
\newblock \emph{Journal of Machine Learning Research}, pages 2153--2175, 2005.

\bibitem[El~Alaoui and Mahoney(2015)]{Alaoui:2015}
A.~El~Alaoui and M.~Mahoney.
\newblock Fast randomized kernel methods with statistical guarantees.
\newblock In \emph{Advances in Neural Information Processing Systems}, 2015.

\bibitem[Gretton et~al.(2008)Gretton, Fukumizu, H., Song, Sch\"{o}lkopf, and
  Smola]{kern_test}
A.~Gretton, K.~Fukumizu, C.~H., L.~Song, B.~Sch\"{o}lkopf, and A.~Smola.
\newblock A kernel statistical test of independence.
\newblock In \emph{Advances in Neural Information Processing Systems}, pages
  585--592, 2008.

\bibitem[Gretton et~al.(2012)Gretton, Borgwardt, Rasch, Sch{\"o}lkopf, and
  Smola]{Gretton:2012}
A.~Gretton, K.~Borgwardt, M.~Rasch, B.~Sch{\"o}lkopf, and A.~Smola.
\newblock A kernel two-sample test.
\newblock \emph{Journal of Machine Learning Research}, pages 723--773, 2012.

\bibitem[Halko et~al.(2011)Halko, Martinsson, and Tropp]{Halko:2011}
N.~Halko, P.~Martinsson, and J.~Tropp.
\newblock Finding structure with randomness: Probabilistic algorithms for
  constructing approximate matrix decompositions.
\newblock \emph{SIAM Review}, pages 217--288, 2011.

\bibitem[Hamid et~al.(2014)Hamid, Xiao, Gittens, and DeCoste]{compact_rf}
R.~Hamid, Y.~Xiao, A.~Gittens, and D.~DeCoste.
\newblock Compact random feature maps.
\newblock In \emph{International Conference on International Conference on
  Machine Learning}, 2014.

\bibitem[Hofmann et~al.(2008)Hofmann, Sch{\"o}lkopf, and Smola]{Hofmann:2008}
T.~Hofmann, B.~Sch{\"o}lkopf, and A.~Smola.
\newblock Kernel methods in machine learning.
\newblock \emph{The Annals of Statistics}, pages 1171--1220, 2008.

\bibitem[Honorio and Li(2017)]{Honorio:2017}
J.~Honorio and Y.-J. Li.
\newblock The error probability of random {Fourier} features is dimensionality
  independent.
\newblock \emph{arXiv:1710.09953}, 2017.

\bibitem[Hsieh et~al.(2008)Hsieh, Chang, Lin, Keerthi, and
  Sundararajan]{lin_svm_train}
C.~Hsieh, K.~Chang, C.~Lin, S.~Keerthi, and S.~Sundararajan.
\newblock A dual coordinate descent method for large-scale linear {SVM}.
\newblock In \emph{International Conference on Machine Learning}, pages
  408--415, 2008.

\bibitem[Huang et~al.(2014)Huang, Avron, Sainath, Sindhwani, and
  Ramabhadran]{dnn_vs_rf}
P.~Huang, H.~Avron, T.~Sainath, V.~Sindhwani, and B.~Ramabhadran.
\newblock Kernel methods match deep neural networks on {TIMIT}.
\newblock In \emph{International Conference on Acoustics, Speech and Signal
  Processing}, pages 205--209, May 2014.

\bibitem[Johnson et~al.(1986)Johnson, Lindenstrauss, and Schechtman]{JL_algo}
W.~Johnson, J.~Lindenstrauss, and G.~Schechtman.
\newblock Extensions of {Lipschitz} maps into {Banach} spaces.
\newblock \emph{Israel Journal of Mathematics}, pages 129--138, 1986.

\bibitem[Kakade and Shakhnarovich(2009)]{Jl_inner_note}
S.~Kakade and G.~Shakhnarovich.
\newblock Lecture notes in large scale learning, 2009.
\newblock URL
  \url{http://ttic.uchicago.edu/~gregory/courses/LargeScaleLearning/lectures/jl.pdf}.

\bibitem[Kar and Karnick(2012)]{dot_prod_kernel}
P.~Kar and H.~Karnick.
\newblock Random feature maps for dot product kernels.
\newblock In \emph{International Conference on Artificial Intelligence and
  Statistics}, pages 583--591, 2012.

\bibitem[Le et~al.(2013)Le, Sarlos, and Smola]{fast_food}
Q.~Le, T.~Sarlos, and A.~Smola.
\newblock Fastfood - approximating kernel expansions in loglinear time.
\newblock In \emph{International Conference on Machine Learning}, 2013.

\bibitem[Li et~al.(2006)Li, Hastie, and Church]{sparse_rand_projs}
P.~Li, T.~Hastie, and K.~Church.
\newblock Very sparse random projections.
\newblock In \emph{International Conference on Knowledge Discovery and Data
  Mining}, pages 287--296, 2006.

\bibitem[Lim et~al.(2018)Lim, Du, Dai, Jung, Song, and Park]{Lim:2018}
W.~Lim, R.~Du, B.~Dai, K.~Jung, L.~Song, and H.~Park.
\newblock Multi-scale {Nystrom} method.
\newblock In \emph{International Conference on Artificial Intelligence and
  Statistics}, 2018.

\bibitem[Marguerite and Philip(1956)]{franke_wolfe}
F.~Marguerite and W.~Philip.
\newblock An algorithm for quadratic programming.
\newblock \emph{Naval Research Logistics Quarterly}, pages 95--110, 1956.

\bibitem[Mendelson(2003)]{Mendelson:2003}
S.~Mendelson.
\newblock On the performance of kernel classes.
\newblock \emph{Journal of Machine Learning Research}, pages 759--771, 2003.

\bibitem[Musco and Musco(2017)]{Musco:2017}
C.~Musco and C.~Musco.
\newblock Recursive sampling for the {Nystr\"om} method.
\newblock In \emph{Advances in Neural Information Processing Systems}, 2017.

\bibitem[Pennington et~al.(2015)Pennington, Yu, and Kumar]{spherical_rf}
J.~Pennington, F.~Yu, and S.~Kumar.
\newblock Spherical random features for polynomial kernels.
\newblock In \emph{Advances in Neural Information Processing Systems}, pages
  1846--1854, 2015.

\bibitem[Rahimi and Recht(2007)]{rahimi_rf}
A.~Rahimi and B.~Recht.
\newblock Random features for large-scale kernel machines.
\newblock In \emph{Neural Information Processing Systems}, 2007.

\bibitem[Rahimi and Recht(2008)]{recht_rf}
A.~Rahimi and B.~Recht.
\newblock Random features for large-scale kernel machines.
\newblock In \emph{Advances in Neural Information Processing Systems}, pages
  1177--1184, 2008.

\bibitem[Rudi and Rosasco(2017)]{rudi_kernel_ridge}
A.~Rudi and L.~Rosasco.
\newblock Generalization properties of learning with random features.
\newblock In \emph{Advances in Neural Information Processing Systems}, 2017.

\bibitem[Rudi et~al.(2015)Rudi, Camoriano, and Rosasco]{Rudi:2015}
A.~Rudi, R.~Camoriano, and L.~Rosasco.
\newblock Less is more: {Nystr\"om} computational regularization.
\newblock In \emph{Advances in Neural Information Processing Systems}, 2015.

\bibitem[Rudin(1994)]{bochner}
W.~Rudin.
\newblock \emph{Fourier Analysis on Groups}, chapter The Basic Theorems of
  {Fourier} Analysis.
\newblock Wiley, 1994.

\bibitem[Samo and Roberts(2015)]{Samo:2015}
Y.~Samo and S.~Roberts.
\newblock Generalized spectral kernels.
\newblock \emph{arXiv:1506.02236}, 2015.

\bibitem[Saunders et~al.(1998)Saunders, Gammerman, and Vovk]{ridge_regres}
C.~Saunders, A.~Gammerman, and V.~Vovk.
\newblock Ridge regression learning algorithm in dual variables.
\newblock In \emph{International Conference on Machine Learning}, pages
  515--521, 1998.

\bibitem[Sch{\"o}lkopf and Smola(2001)]{kern_book}
B.~Sch{\"o}lkopf and A.~Smola.
\newblock \emph{Learning with Kernels: Support Vector Machines, Regularization,
  Optimization, and Beyond}.
\newblock MIT Press, 2001.

\bibitem[Sch{\"o}lkopf et~al.(1997)Sch{\"o}lkopf, Smola, and
  M{\"u}ller]{kern_pca}
B.~Sch{\"o}lkopf, A.~Smola, and K.~M{\"u}ller.
\newblock Kernel principal component analysis.
\newblock In \emph{Artificial Neural Networks}, pages 583--588, 1997.

\bibitem[Shen et~al.(2017)Shen, Yang, and Wang]{moment_rf}
W.~Shen, Z.~Yang, and J.~Wang.
\newblock Random features for shift-invariant kernels with moment matching.
\newblock In \emph{Association for the Advancement of Artificial Intelligence
  Conference}, 2017.

\bibitem[Sriperumbudur and Sterge(2017)]{kern_pca_asym}
B.~Sriperumbudur and N.~Sterge.
\newblock Approximate kernel {PCA} using random features: Computational vs.\
  statistical trade-off.
\newblock \emph{arXiv:1706.06296}, 2017.

\bibitem[Sriperumbudur et~al.(2010)Sriperumbudur, Gretton, Fukumizu,
  Sch{\"o}lkopf, and Lanckriet]{Sriperumbudur:2010}
B.~Sriperumbudur, A.~Gretton, K.~Fukumizu, B.~Sch{\"o}lkopf, and G.~Lanckriet.
\newblock {Hilbert} space embeddings and metrics on probability measures.
\newblock \emph{Journal of Machine Learning Research}, pages 1517--1561, 2010.

\bibitem[Sutherland and Schneider(2015)]{error_rand_feats}
D.~Sutherland and J.~Schneider.
\newblock On the error of random {Fourier} features.
\newblock In \emph{Conference on Uncertainty in Artificial Intelligence}, pages
  862--871, 2015.

\bibitem[Talwalkar(2010)]{thesis}
A.~Talwalkar.
\newblock \emph{Matrix Approximation for Large-scale Learning}.
\newblock PhD thesis, New York University, 2010.

\bibitem[Tibshirani(1994)]{lasso}
R.~Tibshirani.
\newblock Regression shrinkage and selection via the lasso.
\newblock \emph{Journal of the Royal Statistical Society, Series B}, pages
  267--288, 1994.

\bibitem[Vapnik(1998)]{Vapnik:1998}
V.~Vapnik.
\newblock \emph{Statistical Learning Theory}.
\newblock John Wiley {\&} Sons, New York, 1998.

\bibitem[Vapnik et~al.(1997)Vapnik, Golowich, and Smola]{svm}
V.~Vapnik, S.~Golowich, and A.~Smola.
\newblock Support vector method for function approximation, regression
  estimation and signal processing.
\newblock In \emph{Advances in Neural Information Processing Systems}, pages
  281--287, 1997.

\bibitem[Williams and Seeger(2001)]{nyst_seeg}
C.~Williams and M.~Seeger.
\newblock Using the {Nystr\"{o}m} method to speed up kernel machines.
\newblock In \emph{Advances in Neural Information Processing Systems}, pages
  682--688, 2001.

\bibitem[Yang et~al.(2012)Yang, Li, Mahdavi, Jin, and Zhou]{Yang:2012}
T.~Yang, Y.~Li, M.~Mahdavi, R.~Jin, and Z.~Zhou.
\newblock {Nystr{\"o}m} method vs random {Fourier} features - a theoretical and
  empirical comparison.
\newblock In \emph{Advances in Neural Information Processing Systems}, 2012.

\bibitem[Yang et~al.(2017)Yang, Pilanci, and Wainwright]{Yang:2017}
Y.~Yang, M.~Pilanci, and M.~J. Wainwright.
\newblock Randomized sketches for kernels: Fast and optimal nonparametric
  regression.
\newblock \emph{The Annals of Statistics}, pages 991--1023, 2017.

\bibitem[Yu et~al.(2016)Yu, Suresh, Choromanski, Holtmann-Rice, and
  Kumar]{orthog_rf}
F.~Yu, A.~Suresh, K.~Choromanski, D.~Holtmann-Rice, and S.~Kumar.
\newblock Orthogonal random features.
\newblock In \emph{Advances in Neural Information Processing Systems}, pages
  1975--1983, 2016.

\bibitem[Zhang et~al.(2011)Zhang, Peters, Janzing, and Sch\"{o}lkopf]{kern_ci}
K.~Zhang, J.~Peters, D.~Janzing, and B.~Sch\"{o}lkopf.
\newblock Kernel-based conditional independence test and application in causal
  discovery.
\newblock In \emph{Conference on Uncertainty in Artificial Intelligence}, pages
  804--813, 2011.

\end{thebibliography}

\end{document}